\documentclass{article}
\usepackage{iclr2026/iclr2026_conference,times}

\usepackage{amsmath,amssymb,amsthm,mathtools}
\usepackage{microtype}
\usepackage{booktabs}
\usepackage{enumitem}
\usepackage{graphicx}
\usepackage{xcolor}
\usepackage{hyperref}
\usepackage{url}

\newtheorem{theorem}{Theorem}

\newtheorem{proposition}{Proposition}
\newtheorem{corollary}{Corollary}
\newtheorem{definition}{Definition}
\newtheorem{remark}{Remark}

\newcommand{\Sset}{\mathcal S}
\newcommand{\Aset}{\mathcal A}
\newcommand{\Q}{\mathcal Q}

\newcommand{\E}{\mathbb E}

\newcommand{\ip}[2]{\left\langle #1,#2\right\rangle}
\newcommand{\norm}[1]{\left\lVert #1\right\rVert}
\newcommand{\argmin}{\operatorname*{arg\,min}}

\newcommand{\DReg}{\operatorname{DReg}}
\newcommand{\diam}{\operatorname{diam}}

\RequirePackage{xcolor}
\definecolor{KomoriBlack}{HTML}{050505}
\definecolor{KomoriInk}{HTML}{151514}
\definecolor{KomoriGraphite}{HTML}{373633}
\definecolor{KomoriMuted}{HTML}{625E57}
\definecolor{KomoriFaint}{HTML}{928B7F}
\definecolor{KomoriLine}{HTML}{E2DDD3}
\definecolor{KomoriPaper}{HTML}{FBFAF6}
\definecolor{KomoriWarm}{HTML}{F1EDE3}
\definecolor{KomoriAccent}{HTML}{8A6A2E}
\definecolor{KomoriSky}{HTML}{14B8FF}
\definecolor{KomoriBlue}{HTML}{1D7CFF}
\definecolor{KomoriMint}{HTML}{12B981}
\definecolor{KomoriPink}{HTML}{FF4FA3}
\definecolor{KomoriViolet}{HTML}{8B5CF6}
\definecolor{KomoriCite}{HTML}{00C08B}
\definecolor{KomoriHref}{HTML}{18A9FF}
\color{KomoriInk}

\hypersetup{
	colorlinks=true,
	linkcolor=KomoriPink,
	citecolor=KomoriCite,
	urlcolor=KomoriHref,
	filecolor=KomoriViolet,
	pdfborder={0 0 0}
}

\title{Priced Motion Through Optimal Faces: A Normal-Fan Geometry for Non-Stationary Adversarial MDPs}

\author{%
	Kai Hidajat \\
	Komori Research \\
	\texttt{hidajat@komori.ai}
}

\iclrfinalcopy 

\begin{document}

\maketitle

\begin{abstract}
In a changing decision problem, standard dynamic-regret analyses have often equated the cost of non-stationarity to how far loss moves. However, it is simultaneously possible for a loss sequence to travel far and retain the same optimal policy, or for a small movement in loss to force the optimal policy to change completely. Thus, the size of the movement through loss variation, transition variation, or comparator path length describe the adversary's motion, but not the cost of that motion to the control problem. For a more faithful analytic interpretation, this paper develops a normal-fan geometry for finite-horizon adversarial MDPs with fixed transitions. Occupancy measures form a polytope, and each loss vector exposes an optimal face of that polytope. Non-stationarity in rewards is therefore a path through the normal fan, where motion inside one cone leaves the optimal face unchanged, while crossing a wall may carry regret. We pose the notion of a face-crossing price, which is the minimum regret incurred by remaining on the previous optimal face under the new loss. For any learner that tracks the previous face, dynamic regret decomposes exactly into intrinsic priced face motion plus within-face selection error. The resulting theory separates consequential from harmless non-stationarity, where loss variation can be arbitrarily large at zero price, and identical one-coordinate variation can hide horizon-scale differences in regret. 
\end{abstract}

\section{Introduction}

Consider the Recommender system \citep{Schafer_Konstan_etal_1999} which gives personalised recommendations to users based on learned user preferences. The user's interaction with this system is sequential, and with each user action, the
Recommender system decides what to recommend. We can formulate this problem as a stationary Markov Decision Process (MDP) \citep{Puterman_2014,Shani_Brafman_etal_2015}, where we try to learn the user's preference (\(r\)) as well as how their interactions tend to evolve (\(P(s^\prime \mid s, a)\)). However, a user's preference is ever-changing, and we cannot assume that an identical state-action should yield the same reward from one episode to the next \citep{Huleihel_Pal_etal_2021}.

In the fixed-transition non-stationary RL setting, the optimal policy drifts as the rewards and losses change. Existing analyses bound dynamic regret by budgeting the variation of the losses or the comparator path \citep{Besbes_Gur_etal_2015,Zhao_Zhang_etal_2024,Fei_Yang_etal_2020,Zhao_Li_etal_2022}, but these metrics are external to the control problem. That is, such budgets only capture the magnitude of change, and such magnitudes alone do not determine the cost a learner pays.

As an example, we can construct two cases in which loss displacement is inversely proportional to cost. In one, a loss travels far yet keeps the same policy optimal, meaning that a stationary learner pays nothing. In the other, a small step crosses a decision boundary and flips the optimal policy to the far side of the problem, costing the stationary learner regret of order the horizon. For the Recommender system case, suppose two film recommendation policies share the same regret, where one recommends documentaries and one recommends romcoms. Inside the ``documentaries still best'' region the optimal policy does not change, and a large reward movement there barely costs anything. Conversely, a tiny preference shift can move the user across the indifference boundary, making the old ``documentaries still best'' policy systematically wrong, and causing large regret if the learner stays stationary.

The occupancy representation of an MDP gives us a geometric formalisation of decision regions \citep{Puterman_2014,Altman_2021}. For a fixed transition kernel, each policy induces an occupancy measure, the vector of expected visitation frequencies over state-action pairs. There are a few properties which make occupancies solutions to the aforementioned discrepancy in variation. First, policies that share an occupancy attain equal value \(V^\pi(\ell)\) under every loss, making the occupancy a sufficient statistic for value. Next, expected loss is then linear in this vector. Consequently, the set of policies can be represented by a convex occupancy polytope, and the region of losses for which a given policy is optimal is a normal cone of that polytope. Rotating the loss leaves that optimal face fixed across the normal cone, changing only when the loss crosses a wall. The collection of such cones which together correspond to all the unique faces of the occupancy polytope is called the normal fan, which is the map carrying each loss to its optimal face \citep{Ziegler_1995,Rockafellar_1970,Lu_Robinson_2008}.

While loss remains within a cone and the optimal face stays fixed, a learner who holds the face pays no regret regardless of loss variation. A cost arises only when the loss crosses into a cone whose optimal face excludes the learner's occupancy. In such a case, holding the previous optimal face under the new loss then costs the least regret achievable over its points. Thus, by searching through the entire fan, we are able to exactly price the motion of losses. However, this search is a combinatorial problem, and not tractable to enumerate. In this paper, we show that no search is ever required. The price equals an expected optimal Bellman advantage under the current loss, measured against the optimal value function that loss induces. The advantage is accumulated along whichever policy was optimal one round earlier. A single value backup supplies these advantages, and the remaining minimization runs only over the previous optimal face, which the learner already holds. 

\begin{figure}[t]
	\centering
	\includegraphics[width=0.75\linewidth]{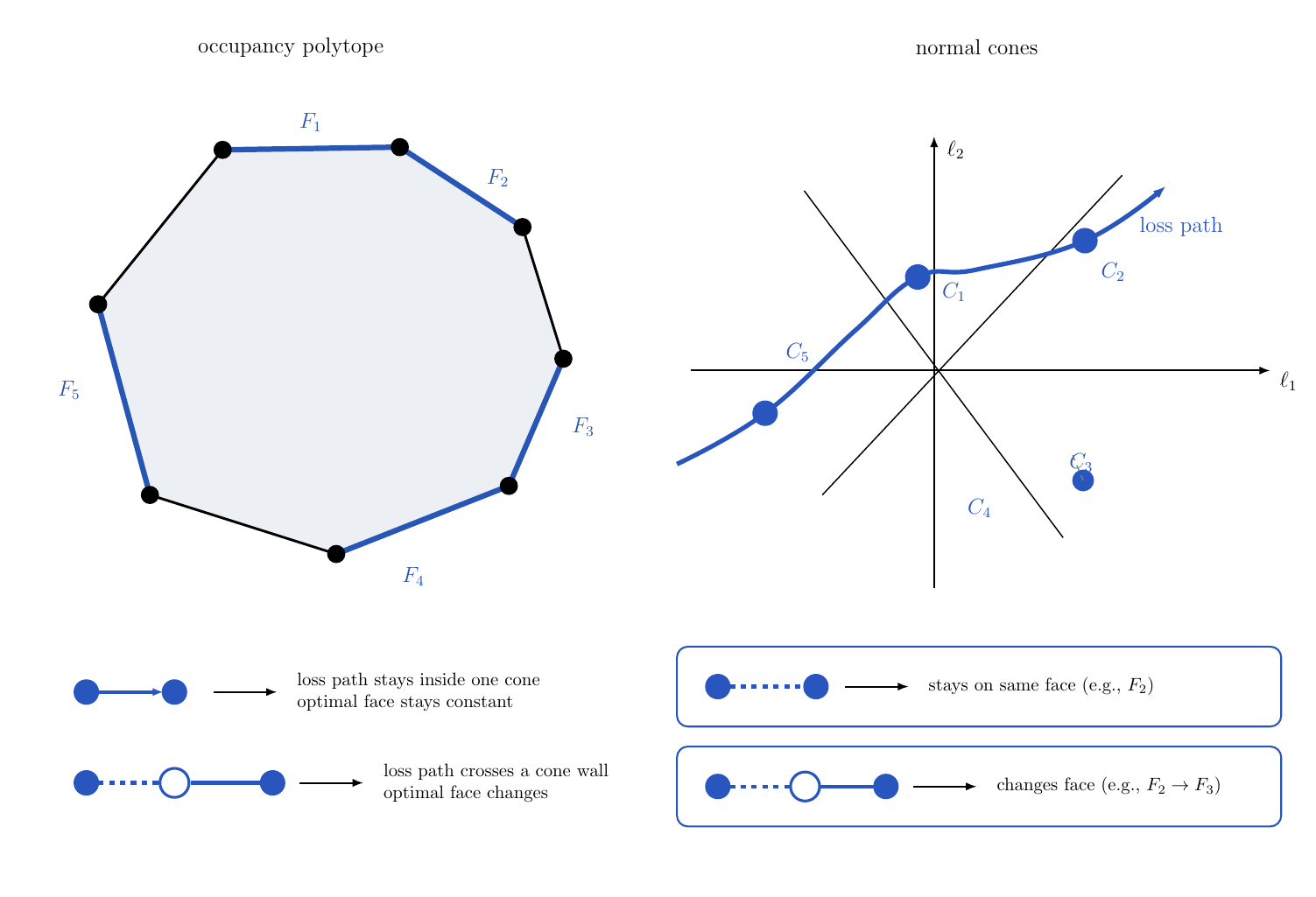}
	\caption{Occupancy faces and normal cones. Left, the occupancy polytope, whose faces are the candidate sets of optimal occupancies. Right, the normal fan in loss space, which partitions losses into cones by the face each one selects. A loss path inside one cone leaves the optimal face unchanged and is free, while a path that crosses a wall changes the optimal face and may be expensive.}
	\label{fig:normal-fan}
	\vspace{-1em}
\end{figure}


\paragraph{Contributions.}
\begin{enumerate}[leftmargin=1.5em,itemsep=2pt,topsep=2pt]
	\item We define the face-crossing price, the cost of remaining on the previous optimal face under the new loss, and prove that for a learner that holds that face, dynamic regret splits exactly into this price and a within-face selection error (Theorem~\ref{thm:decomp}).
	\item We prove that in any finite-horizon MDP the price equals a minimal expected optimal Bellman advantage over the policies that were optimal a round earlier (Theorem~\ref{thm:bellman}), so a value backup over the old face computes it directly.
	\item We establish three consequences. First, loss variation can be arbitrarily large at zero price, and at fixed variation it can hide a factor $H$ (Proposition~\ref{prop:largevar} and Theorem~\ref{thm:varsep}). Second, a crossing forced early in the horizon costs more than a late one by the factor $H-h+1$ (Corollary~\ref{cor:causal}). Lastly, a degenerate optimal face leaves a selection error the price cannot absorb (Proposition~\ref{prop:degsel}).
\end{enumerate}
A full-information cone detector (Proposition~\ref{prop:conedetect}) and a secondary reading of mirror-descent and trust-region updates as approximate trackers of the fan appear in the appendices. We test our theory on three empirical benchmarks in Section~\ref{sec:experiments}.

\section{Related work}
\label{sec:related}

\paragraph{Occupancy measures and the geometry of control.}
With known transitions, the occupancy measure turns policy optimization into linear optimization over a polytope. This correspondence descends from the linear programming view of dynamic programming \citep{Puterman_2014,Altman_2021}. Online and adversarial MDP algorithms exploit this view directly, through relative-entropy policy search and reductions to online convex optimization over the polytope \citep{EvenDar_Kakade_etal_2009,Neu_Antos_etal_2010,Zimin_Neu_2013,Rosenberg_Mansour_2019}. A parallel line studies the geometry that this representation imposes on policy space, through the value-function polytope \citep{Dadashi_Taiga_etal_2019}, the convex-analytic structure of constrained and inverse RL \citep{Schlaginhaufen_Kamgarpour_2023}, and the occupancy-manifold reading of actor-critic methods \citep{Milosevic_Scherf_2025,Muller_Montufar_2024}. The present geometry is built to measure how those exposed faces move and what their motion costs.

\paragraph{Dynamic regret and non-stationarity.}
Dynamic regret compares a learner against a moving comparator and bounds the gap by a measure of the comparator's motion, an idea that originates in online convex optimization \citep{Zinkevich_2003,Hall_Willett_2013,Jadbabaie_Rakhlin_etal_2015,Eshraghi_Liang_2022,Zhao_Zhang_etal_2020,Zhao_Zhang_etal_2024} and non-stationary stochastic optimization \citep{Besbes_Gur_etal_2015}. The idea has been adopted by non-stationary and adversarial MDPs \citep{Cheung_Simchi-Levi_etal_2020,Fei_Yang_etal_2020,Zhao_Li_etal_2022,Li_Zhao_etal_2023,Li_Zhao_etal_2024,Dick_Gyorgy_etal_2015,Cardoso_Wang_etal_2019}. These results bound regret by comparator path length or by reward and transition variation. The occupancy path-length bounds nearest to ours \citep{Zhao_Li_etal_2022,Li_Zhao_etal_2024} measure non-stationarity by the movement of the optimal occupancy and pay a Lipschitz factor to turn that movement into regret. The price studied here refines that movement to a path of exposed faces measured in value. For a learner holding the previous optimal face, this price equals the realized regret exactly (Theorem~\ref{thm:decomp}), whereas path length only bounds it.

\paragraph{Normal fans, polyhedral geometry, and policy updates.}
The normal fan of a polytope partitions objective space by the face a linear objective selects --- a standard construction in convex and polyhedral geometry \citep{Ziegler_1995,Rockafellar_1970,Lu_Robinson_2008,Paffenholz_2010,Tropp_2022}. We read its walls as the decision boundaries of control. A crossing becomes costly once the previous optimal face acquires positive Bellman advantage under the new loss. The margin around a wall is a weak-sharpness constant of the linear program \citep{Burke_Ferris_1993}. Separately, mirror-descent, natural-gradient, and trust-region methods place a non-Euclidean geometry on policy space \citep{Nemirovskii_IUdin_1983,Beck_Teboulle_2003,Kakade_2001,Kakade_Langford_2002,Schulman_Levine_etal_2017,Tomar_Shani_etal_2021,Lan_2022,Zhan_Cen_etal_2023,Alfano_Yuan_etal_2024,Agarwal_Kakade_etal_2020,Kleuker_Plaat_etal_2025}. Our reading recasts these updates as approximate trackers of the fan, with the Bregman term holding mass near uncertain walls, and a large advantage pushing it off a clearly suboptimal action. The prediction such tracking needs is supplied by the optimistic mechanism of \citet{Rakhlin_Sridharan_2014}.

\section{Occupancy geometry and the normal fan}
\label{sec:geometry}

\subsection{The occupancy polytope}

We consider an episodic MDP with horizon $H$, finite state space $\Sset$, finite action space $\Aset$, initial distribution $\rho$, and transition kernels $P_h(\cdot\mid s,a)$. We take the transitions known and fixed and let the losses $\ell_t=\{\ell_{t,h}(s,a)\}_{h,s,a}$ vary adversarially across episodes. Following the standard occupancy-measure formulation \citep{Puterman_2014,Altman_2021}, for a Markov policy $\pi$ the occupancy measure records the visitation frequencies,
\[
	q_h^\pi(s,a)=\Pr_\pi(s_h=s,\,a_h=a).
\]
Let $\Q$ denote the set of all occupancy measures attainable by some Markov policy. It is the polytope defined by nonnegativity together with the flow constraints
\begin{align*}
	\sum_a q_1(s,a)&=\rho(s),\\
	\sum_a q_{h+1}(s,a)&=\sum_{s',a'} q_h(s',a')\,P_h(s\mid s',a'),\qquad h=1,\ldots,H-1.
\end{align*}
The expected loss is linear in the occupancy, $J_t(\pi)=\ip{q^\pi}{\ell_t}$, and the set of optimal occupancies at round $t$ is the exposed face
\[
	F_t=\argmin_{q\in\Q}\ip{q}{\ell_t}.
\]
For a sequence of plays $x_t\in\Q$ the dynamic regret against the per-round optimum \citep{Zinkevich_2003,Hall_Willett_2013,Jadbabaie_Rakhlin_etal_2015} is
\[
	\DReg_T=\sum_{t=1}^T\left[\ip{x_t}{\ell_t}-\min_{q\in\Q}\ip{q}{\ell_t}\right].
\]

Notice that expected loss sees a policy only through its occupancy measure. Two policies with the same occupancy therefore collapse to one point of $\Q$, and flow conservation makes this quotient a bijection with $\Q$. Their natural distance is then the occupancy distance $d_{\rm occ}([\pi],[\pi'])=\norm{q^\pi-q^{\pi'}}$. Norm duality then equates it with the largest value gap a bounded adversary can open between them, $\sup_{\norm{\ell}_*\le1}|J_\ell(\pi)-J_\ell(\pi')|$.

\subsection{The normal fan}

Given the polytope and its faces, the next question is which losses select which face. To each face $F\subseteq\Q$ let's associate a cone of losses,
\[
	N(F)=\bigl\{\ell:\ F\subseteq \argmin_{q\in\Q}\ip{q}{\ell}\bigr\}.
\]
These cones collect into the (negative) normal fan of $\Q$, and the fan partitions loss space by the face that linear optimization selects \citep{Ziegler_1995,Rockafellar_1970,Lu_Robinson_2008}. For every loss $\ell$ the optimal set $F(\ell)=\argmin_{q\in\Q}\ip{q}{\ell}$ is an exposed face of $\Q$, attained where a supporting hyperplane meets it. A loss in the relative interior of $N(F)$ then selects exactly $F$.

\section{The face-crossing price}
\label{sec:priced}

\subsection{Definition and the regret decomposition}

Suppose the learner enters round $t$ playing some occupancy in the previous optimal face $F_{t-1}$. The unavoidable cost of remaining on that face is the smallest regret achievable by any of its points under the new loss.

\begin{definition}[Minimal face-crossing price]
	For consecutive optimal faces $F_{t-1}$ and $F_t$,
	\[
		\Gamma_t^{\mathrm{cross}}=\min_{u\in F_{t-1}}\left[\ip{u}{\ell_t}-\min_{q\in\Q}\ip{q}{\ell_t}\right],
		\qquad
		\Gamma_{2:T}^{\mathrm{cross}}=\sum_{t=2}^T\Gamma_t^{\mathrm{cross}}.
	\]
\end{definition}

The price is zero when some point of the old face is still optimal under the new loss, and positive only when every point of the old face has become suboptimal. Dynamic regret now decomposes around this quantity.

\begin{theorem}[Exact face-regret decomposition]
	\label{thm:decomp}
	Suppose that for $t\ge2$ the learner plays $x_t\in F_{t-1}$. Then the per-round regret splits as
	\[
		r_t:=\ip{x_t}{\ell_t}-\min_{q\in\Q}\ip{q}{\ell_t}=\Gamma_t^{\mathrm{cross}}+\varepsilon_t^{\rm sel},
		\qquad
		\varepsilon_t^{\rm sel}=\ip{x_t}{\ell_t}-\min_{u\in F_{t-1}}\ip{u}{\ell_t},
	\]
	and consequently
	\[
		\DReg_T=r_1+\sum_{t=2}^T\Gamma_t^{\mathrm{cross}}+\sum_{t=2}^T\varepsilon_t^{\rm sel}.
	\]
\end{theorem}

\begin{proof}
	Add and subtract the best value attainable inside the old face. Writing $m_t=\min_{u\in F_{t-1}}\ip{u}{\ell_t}$,
	\[
		r_t=\Bigl[m_t-\min_{q\in\Q}\ip{q}{\ell_t}\Bigr]+\Bigl[\ip{x_t}{\ell_t}-m_t\Bigr]=\Gamma_t^{\mathrm{cross}}+\varepsilon_t^{\rm sel}.
	\]
	The first bracket is the least cost of remaining on $F_{t-1}$, and the second is the learner's error in choosing its representative within $F_{t-1}$. Summing over $t$ gives the claim.
\end{proof}

The price $\Gamma_t^{\mathrm{cross}}$ is intrinsic, and the loss sequence alone fixes it, beyond the reach of any learner that stays on the old face. By contrast, the selection error $\varepsilon_t^{\rm sel}$ is algorithmic, and it vanishes for a learner that always plays the best point of that face. This split presumes the learner holds the old face throughout the round. Once the learner moves off it, $\varepsilon_t^{\rm sel}$ can turn negative and the two terms no longer separate. 


\subsection{Comparison with loss variation}

The face-crossing price explains when the familiar variation bounds \citep{Besbes_Gur_etal_2015,Zhao_Zhang_etal_2024,Zhao_Li_etal_2022} are tight and when they are wasteful. Consider first the non-degenerate case in which each face is a single vertex.

Take every optimal face to be a single vertex, $F_t=\{v_t\}$, and let the tracker play $x_t=v_{t-1}$. The selection error of Theorem~\ref{thm:decomp} then vanishes, leaving $\DReg_T=r_1+\sum_{t=2}^T\ip{v_{t-1}-v_t}{\ell_t}$. Optimality of $v_{t-1}$ under $\ell_{t-1}$ bounds each term by $\norm{v_{t-1}-v_t}\,\norm{\ell_t-\ell_{t-1}}_*$. Summing recovers the classical variation estimate $\DReg_T\le r_1+\sum_{t=2}^T\norm{v_{t-1}-v_t}\,\norm{\ell_t-\ell_{t-1}}_*$. This estimate charges for every change in the loss, whereas $\Gamma_t^{\mathrm{cross}}$ charges only the part that displaces the optimal face. The next proposition shows that this gap grows without bound.

\begin{proposition}[Large variation, zero price]
	\label{prop:largevar}
	There exist loss sequences with arbitrarily large variation $V_T^\ell=\sum_t\norm{\ell_t-\ell_{t-1}}_*$ and yet $\Gamma_{2:T}^{\mathrm{cross}}=0$.
\end{proposition}

\begin{proof}
	Let every loss lie in the same cone $N(F)$, for instance by adding an arbitrary common offset to a fixed loss or by moving within the relative interior of the cone. The selected face is always $F$, so $F_{t-1}=F_t$ and $\Gamma_t^{\mathrm{cross}}=0$ throughout, while the variation grows without bound as the loss travels inside the cone.
\end{proof}

Whether a crossing is priced depends only on whether the old and new optimal faces still meet (Appendix Figure~\ref{fig:degenerate}). We analyse the complementary case, i.e., an unpredictable crossing that forces a fixed price on every learner with a matching minimax lower bound, in Appendix~\ref{app:lower}.

Loss variation cannot see where in the horizon a crossing falls. At fixed variation it can therefore hide a factor of $H$ in the price. More precisely:

\begin{theorem}[Loss variation cannot see the crossing layer]
	\label{thm:varsep}
	For every horizon $H\ge2$ there exist two finite-horizon MDPs with fixed transitions, each carrying a two-round loss sequence whose round-to-round change touches a single coordinate of the same magnitude. The loss variation $V_T^\ell$ is then identical across the two instances under every $L_p$ norm, while their face-crossing prices differ by the factor $H-1$. Both old faces are vertices, so the selection error of Theorem~\ref{thm:decomp} vanishes and the previous-face tracker's dynamic regret equals the price on each. No bound of the form $\DReg_T\le B(V_T^\ell)$ can then be tight on both, and one of the two is loose by a factor $\Theta(H)$.
\end{theorem}

The transition structure drives the amplification, nudging a single step in the cheap instance while rerouting the entire horizon in the expensive one. The variation itself stays put, untouched by the transitions on either side. Within a single fixed MDP a one-coordinate change cannot open a factor-$H$ gap. Two distinct MDPs are therefore necessary, and Appendix~\ref{app:lower} gives the construction and proof. The separation also speaks only to loss or reward variation. Comparator path length, by contrast, does register the crossing-layer difference.

\subsection{Choosing a representative inside a degenerate face}

The comparison so far has assumed each optimal face is a single vertex. When $F_{t-1}$ is instead a higher-dimensional face, the learner must still choose one occupancy inside it, and that choice is the source of the selection error $\varepsilon_t^{\rm sel}$. A learner that predicts the coming loss can keep this error small.

Predicting the coming loss by $m_t$ and selecting $x_t\in\argmin_{u\in F_{t-1}}\ip{u}{m_t}$ keeps the error at $\varepsilon_t^{\rm sel}\le\diam(F_{t-1})\,\delta_t$, with $\delta_t=\norm{m_t-\ell_t}_*$. A sticky Bregman selector trades this against information held along the directions the old face leaves undetermined, at an added cost $\lambda R_t^\psi$ (Propositions~\ref{prop:predsel}--\ref{prop:bregsel}, Appendix~\ref{app:proofs}). This choice of representative is exactly where the Bregman geometry of mirror descent \citep{Nemirovskii_IUdin_1983,Beck_Teboulle_2003} enters. We discuss this further in Section~\ref{sec:algorithms}.

\section{Bellman advantage as the price of motion}
\label{sec:bellman}

So far the face-crossing price has been a polyhedral quantity, defined by minimization over a face of a complicated polytope. In an MDP it also admits a dynamic-programming form \citep{Puterman_2014}. To exhibit it, fix a loss $\ell_t$ and define the optimal value and action-value functions by the backups
\begin{align*}
	V_{t,H+1}^*(s)&=0,\\
	Q_{t,h}^*(s,a)&=\ell_{t,h}(s,a)+\E_{s'\sim P_h(\cdot\mid s,a)}V_{t,h+1}^*(s'),\\
	V_{t,h}^*(s)&=\min_a Q_{t,h}^*(s,a),
\end{align*}
together with the loss advantage $A_{t,h}^*(s,a)=Q_{t,h}^*(s,a)-V_{t,h}^*(s)\ge 0$. The advantage measures how much worse action $a$ is at step $h$ in state $s$ than acting optimally onward, and it is nonnegative under the minimization convention. These two expressions for the price now connect through a single identity.

\begin{theorem}[Bellman advantage representation]
	\label{thm:bellman}
	For any Markov policy $\pi$,
	\[
		J_t(\pi)-J_t^*=\sum_{h=1}^H\E_\pi\bigl[A_{t,h}^*(s_h,a_h)\bigr],
	\]
	and consequently the face-crossing price is the minimal accumulated advantage over the policies that were optimal a round earlier,
	\[
		\Gamma_t^{\mathrm{cross}}=\min_{\pi\in\Pi_{t-1}^*}\sum_{h=1}^H\E_\pi\bigl[A_{t,h}^*(s_h,a_h)\bigr].
	\]
\end{theorem}

\begin{proof}
	The Bellman identity $Q_{t,h}^*(s_h,a_h)=\ell_{t,h}(s_h,a_h)+\E[V_{t,h+1}^*(s_{h+1})\mid s_h,a_h]$ gives, after taking expectations under $\pi$ and summing over the horizon, a telescoping of the value terms,
	\[
		\sum_{h=1}^H\E_\pi\bigl[A_{t,h}^*(s_h,a_h)\bigr]
		=\sum_{h=1}^H\E_\pi\bigl[Q_{t,h}^*(s_h,a_h)-V_{t,h}^*(s_h)\bigr]
		=J_t(\pi)-\E_{s_1\sim\rho}V_{t,1}^*(s_1)
		=J_t(\pi)-J_t^*.
	\]
	Minimizing the left side over the policies optimal at round $t-1$, whose occupancies are exactly the points of $F_{t-1}$, returns the face-crossing price.
\end{proof}

The first identity is the performance-difference lemma \citep{Kakade_Langford_2002} taken against the optimal policy, which attains $Q^*$ and so carries the optimal advantage. The second equality is our contribution. It identifies the polyhedral face-crossing price with this restricted minimization and prices a crossing in the normal fan without ever building it. A single value backup produces the advantages $A_t^*$ at once. The remaining minimization runs over $\Pi_{t-1}^*$, the policies optimal a round earlier. In a finite-horizon MDP these are read off directly, the policies supported on the previous loss's zero-advantage actions, with no search required. The minimization collapses to one trajectory expectation when that earlier optimum is unique. Figure~\ref{fig:bellman-causal} (left) shows the price assembled from per-step advantages along the policies of the previous optimal face.

\paragraph{Causal anisotropy.}
The representation weights a crossing by the number of steps its error affects. A wrong action only at the final step pays $\gamma$ once, affecting nothing downstream. Meanwhile, a wrong first action that then errs by margin $\gamma$ at all $H$ steps pays $H\gamma$. A deterministic layered MDP turns this into a closed form.

\begin{corollary}[Causal anisotropy]
	\label{cor:causal}
	In a deterministic layered MDP, a crossing forced at layer $h$ with per-step margin $\gamma$ has price $\Gamma_t^{\mathrm{cross}}=(H-h+1)\gamma$, so $\Gamma_t^{\mathrm{cross}}/\gamma=H-h+1$.
\end{corollary}

The right panel of Figure~\ref{fig:bellman-causal} illustrates this in a layered MDP.

\begin{figure}[t]
	\centering
	\begin{minipage}[c]{0.53\linewidth}\centering
		\includegraphics[width=\linewidth]{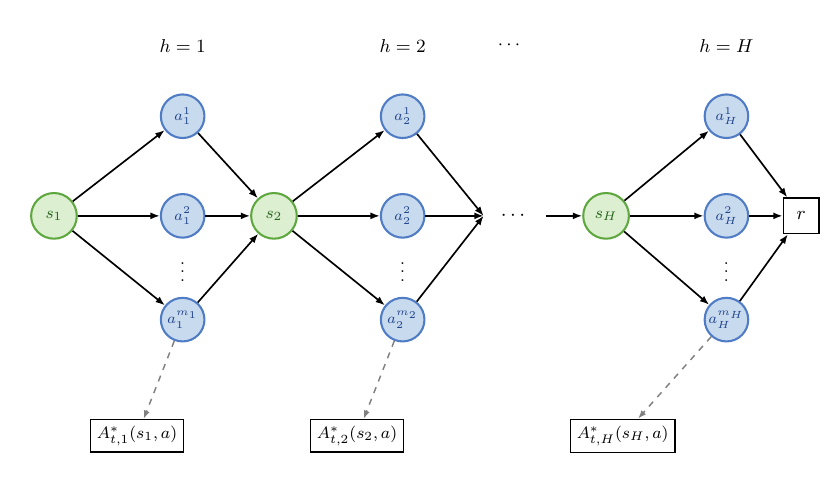}
	\end{minipage}\hfill
	\begin{minipage}[c]{0.41\linewidth}\centering
		\includegraphics[width=\linewidth]{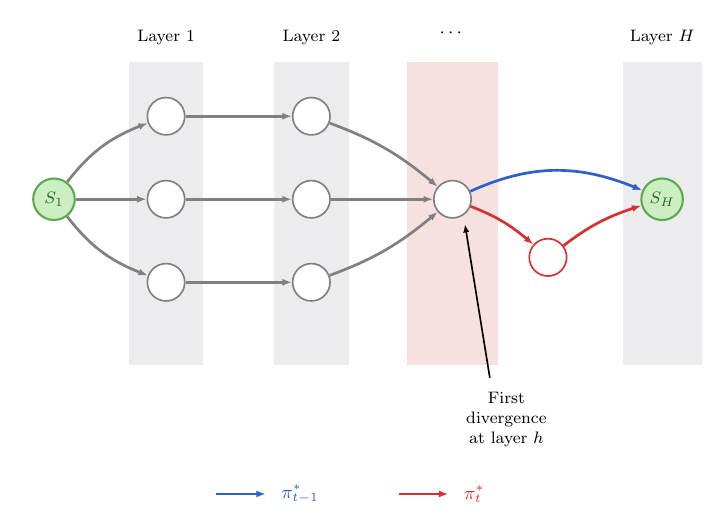}
	\end{minipage}
	\caption{The Bellman price and its causal consequence. Left, under a fixed new loss the optimal advantage $A_{t,h}^*(s,a)$ is computed by a value backup, and the face-crossing price is the smallest expected advantage accumulated along a policy that was optimal under the previous loss (Theorem~\ref{thm:bellman}). Right, in a layered MDP a divergence between the old and new optimal policies at an early layer redirects all downstream visitation, so the price grows with the number of affected steps.}
	\label{fig:bellman-causal}
\end{figure}

\section{Detecting cones without computing the fan}
\label{sec:algorithms}
Pricing a crossing assumes that the crossing has already happened. A learner must act before one arrives, so it has to locate itself in the fan. That fan is far too large to enumerate beyond small examples. The optimal advantages under the current loss already give that location. Each one, produced by the backup of Section~\ref{sec:bellman}, measures how far an action sits from a wall of the current optimal cone. A learner thus reads its position in the fan from a quantity it has already computed.

\subsection{The advantage as a cone detector}

Under full information a prediction error matters only when it can carry the learner across a wall. The relevant scale is the cone margin $\Delta_t$, the smallest suboptimality among vertices outside $F_t$ and a weak-sharpness constant of the linear program \citep{Burke_Ferris_1993}. Whenever the model's uniform value error $\varepsilon_t$ stays below $\Delta_t/2$, the round contributes no regret (Proposition~\ref{prop:conedetect}, Appendix~\ref{app:proofs}). A learner thus pays for its prediction error only on the rounds whose margin is thin.

\subsection{Advantage-thresholded mirror descent}

The same principle suggests how an update rule should move mass. At round $t$ the learner forms a model $m_t$ of the loss, computing predicted advantages $\widehat A_{t,h}(s,a)$ with confidence radii $\beta_{t,h}(s,a)$. The model marks the actions that are statistically indistinguishable from optimal,
\[
	\Aset_{t,h}^{\rm near}(s)=\bigl\{a:\ \widehat A_{t,h}(s,a)\le c\,\beta_{t,h}(s,a)\bigr\}.
\]
It then takes a mirror-descent step that is reluctant to abandon the near set but free to leave the rest,
\[
	\pi_{t+1}=\argmin_{\pi}\Bigl\{\eta_t\,\widehat J_t(\pi)+D_\psi(\pi,\pi_t)+\lambda_t\sum_{h,s}\sum_{a\notin\Aset_{t,h}^{\rm near}(s)}\pi_h(a\mid s)\Bigr\}.
\]
We intend this as an interpretation of existing methods rather than a new algorithm, and keep the account informal. The Bregman term $D_\psi(\pi,\pi_t)$ holds probability mass near the uncertain walls, as the Bregman selector of Proposition~\ref{prop:bregsel} does. The penalty then moves mass off actions whose advantage is confidently positive. Read this way, natural-gradient, trust-region, and policy-mirror-descent updates \citep{Kakade_2001,Schulman_Levine_etal_2017,Tomar_Shani_etal_2021,Lan_2022} should track the fan well, holding mass at thin-margin walls and leaving high-advantage actions quickly. We build the model $m_t$ optimistically from past losses \citep{Rakhlin_Sridharan_2014}, supplying the prediction that Proposition~\ref{prop:predsel} needs. Appendix Figure~\ref{fig:advantage-gate} illustrates the resulting threshold.

Proposition~\ref{prop:gate}, proved in Appendix~\ref{app:gate}, bounds the tracker's regret on the rounds the prediction resolves, and pins down what the penalty alone changes. On every round whose margin the prediction resolves, with $2\varepsilon_t<\Delta_t$, the tracker adds nothing to the intrinsic price beyond the controllable selection term. On a single state-step with a uniform per-layer margin, a fixed step size lets the penalty suppress a confidently suboptimal action at layer $h$ faster than a plain mirror step, by the causal factor $1+\lambda(H-h+1)/\beta$ of Corollary~\ref{cor:causal}. An adversary sitting on a wall instead drives $\Delta_t\to0$ and forces the $\Omega(\gamma T)$ of Proposition~\ref{prop:lower}. This uniform-margin reading therefore holds only for benign non-stationarity.

\section{Experiments}
\label{sec:experiments}

Three constructed environments verify the paper's three central predictions. For the experiments, we keep the normal fan small, so the face price $\Gamma_t^{\mathrm{cross}}$, the crossing layer, and the selection error all take closed forms that we hold against the measured regret. All numbers below average over $32$ seeds.

\paragraph{Loss variation is not the difficulty (simplex).}
The previous-face tracker's regret is a linear function of the face price $\Gamma_{2:T}^{\mathrm{cross}}$, at sample correlation $1.00$ against only $0.89$ for the loss variation $V_T^\ell$ (Appendix Figure~\ref{fig:exp-simplex}). Theorem~\ref{thm:decomp} predicts this exact linearity for a tracker that holds a vertex face. The gap shows most sharply in the within-cone regime. There the loss travels far, at mean variation $2.6\times10^3$, while the optimal face never moves. The price $\Gamma_{2:T}^{\mathrm{cross}}$ stays at zero, and the tracker pays almost no regret. The environment is a one-state problem with $K$ actions, which makes $\Q$ the simplex, and we drive the loss through four regimes. The clean relationship belongs to a learner that holds the face. Taken over all four algorithms at once, both correlations fall near $0.46$, the geometry-blind methods contributing regret tied to neither quantity.

\paragraph{Early crossings cost more (layered tree).}
The measured ratio $\Gamma_t^{\mathrm{cross}}/\gamma$ lands on the predicted identity $H-h+1$ of Corollary~\ref{cor:causal} to within floating-point error. Across horizons up to $20$ the largest deviation is $2\times10^{-15}$. The tree therefore confirms the closed-form identity rather than estimating it (Figure~\ref{fig:exp-ab}, left). An early crossing is therefore $H$ times as expensive as a late one. The instance is a deterministic binary tree, each policy a root-to-leaf path, and we force a crossing at a chosen layer $h$.

\paragraph{Regret splits into price and selection error (degenerate faces).}
The benchmark makes the optimal set a prefix face. Every selector inside it then pays the same intrinsic price, and only the representative differs. We compare four selectors --- an oracle that sees $\ell_t$ and plays the best point of $F_{t-1}$, a sticky Bregman selector that stays near its previous choice (Proposition~\ref{prop:bregsel}), and lexicographic and random tie-breaks (defined in Appendix~\ref{app:experiments}). All four pay the identical price of Theorem~\ref{thm:decomp}, and the regret above it is pure selection error. The oracle adds none, while the sticky selector adds per-round error $0.062$, lexicographic adds $0.109$, and random adds $0.110$ (Figure~\ref{fig:exp-ab}, right). The cumulative-regret stack reproduces the decomposition directly.

\begin{figure}[t]
	\centering
	\begin{minipage}[t]{0.46\linewidth}\centering
		\includegraphics[width=\linewidth]{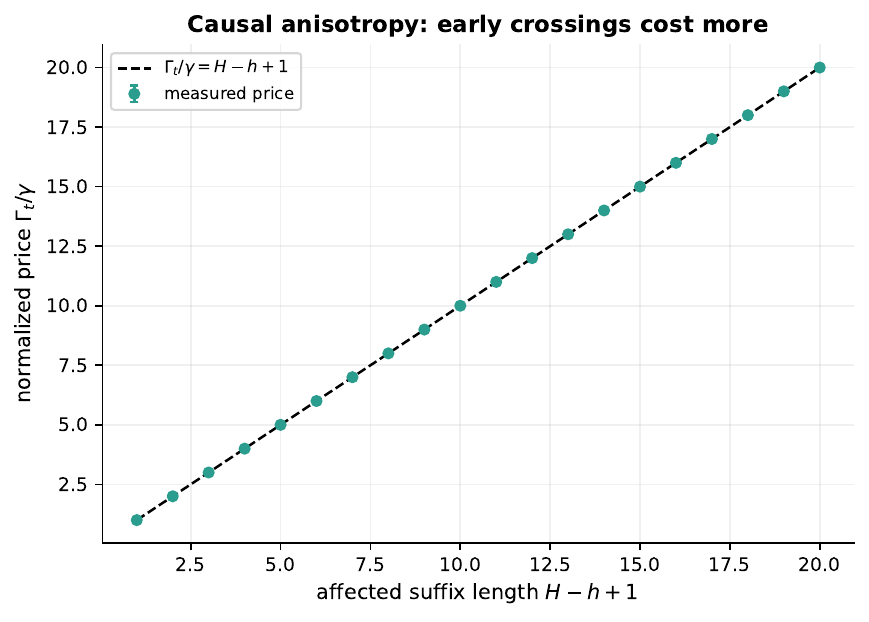}
	\end{minipage}\hfill
	\begin{minipage}[t]{0.52\linewidth}\centering
		\includegraphics[width=\linewidth]{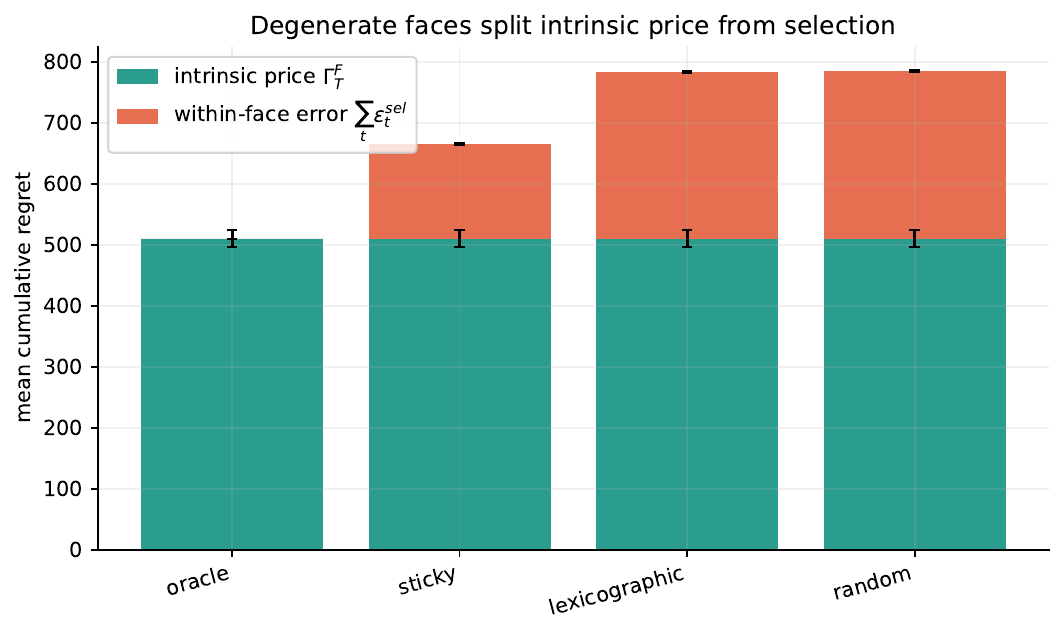}
	\end{minipage}
	\caption{Left, the causal-anisotropy identity in the layered tree, where the measured $\Gamma_t^{\mathrm{cross}}/\gamma$ against $H-h+1$ falls on the identity line to within $2\times10^{-15}$. Right, the regret decomposition on degenerate faces, where the intrinsic price is the same for every selector and the regret above it is selection error, which the oracle drives to zero and the sticky selector keeps below the lexicographic and random ones. Error bars show standard errors.}
	\label{fig:exp-ab}
	\vspace{-1em}
\end{figure}

\paragraph{Heuristic update (secondary).}
As a consistency check on Section~\ref{sec:algorithms}, we run the advantage-thresholded update --- the penalized mirror-descent step defined there --- on custom non-stationary gridworlds, following the standard tabular gridworld convention in reinforcement learning \citep{Sutton_Barto_2018}. The baselines are mirror descent, optimistic mirror descent, and a trust-region update. The thresholded update attains the lowest cumulative regret in every scenario, beating the mirror-descent baselines by factors of three to four and the best-tuned trust-region method (Appendix Figure~\ref{fig:exp-grid}). Proposition~\ref{prop:gate} attributes the gain to the causal-anisotropy factor $H-h+1$ of Corollary~\ref{cor:causal}. Unlike the three structural results, this gain has no closed-form tie to $\Gamma_{2:T}^{\mathrm{cross}}$.
\vspace{-1em}

\section{Discussion}
\label{sec:discussion}

This paper recasts the cost of non-stationarity in fixed-transition adversarial MDPs as the priced motion of the optimal face through the normal fan. That quantity is at once geometrically intrinsic and computable by a single value backup. The face-crossing price improves on the usual measures of non-stationarity in two concrete ways. It counts only the motion that forces a change of policy. The price stays at zero even as loss variation or comparator path length grows without bound. It is also exact where the usual measures only bound. Theorem~\ref{thm:bellman} writes it as an expected optimal Bellman advantage over the already-known old face. Variation and path length reach regret only through a Lipschitz constant. The decomposition $\DReg=\text{priced face motion}+\text{within-face selection error}$ holds whenever the learner stays on the old face, and it cleanly separates the intrinsic price from the choice of representative inside a degenerate face. Its sharpest consequence is the causal anisotropy $\Gamma_t^{\mathrm{cross}}/\gamma=H-h+1$ of Corollary~\ref{cor:causal}. The price depends on where in the horizon a crossing falls, a dependence loss variation cannot express and the experiments confirm to floating-point precision. For the previous-face tracker the price is two-sided, the decomposition giving an exact upper bound and the $\Omega(\gamma T)$ lower bound of Proposition~\ref{prop:lower} matching it when crossings are unpredictable. Under degeneracy the selection error is an independent degree of freedom (Proposition~\ref{prop:degsel}).

\paragraph{Limitations.}
The structural results assume fixed transitions and a finite horizon. We leave the moving-transition case, studied in non-stationary MDP work through transition variation \citep{Cheung_Simchi-Levi_etal_2020,Fei_Yang_etal_2020,Mao_Zhang_etal_2022,Wei_Luo_2021}, to future work. In such cases, the polytope itself deforms between rounds. Additionally, our reading of mirror-descent and trust-region updates as fan trackers remains heuristic, with no adversarial regret guarantee.

\paragraph{Bandit feedback.}
The most natural extension is to bandit feedback \citep{CesaBianchi_Lugosi_2006,Neu_Antos_etal_2010,Jin_Jin_etal_2019,Li_Zhao_etal_2024a}, where the learner sees only the losses on the trajectory it plays and must estimate the advantage of the unplayed actions. The cone detector of Proposition~\ref{prop:conedetect} needs only a uniform value confidence radius $\sup_\pi|J_{\ell_t}(\pi)-\widehat J_t(\pi)|\le\varepsilon_t^{\rm bandit}$. The same margin test $2\varepsilon_t^{\rm bandit}<\Delta_t$ then decides which rounds are statistically free and which sit near a wall. The open problem is to build that radius from an importance-weighted estimator over a sliding window. Its per-state-action width must track the visitation $N_{t,h}(s,a)$, the horizon, and the window length, and must aggregate to a uniform value bound concentrated at the thin-margin walls of the fan. Such a radius would turn the heuristic split into a genuine regret bound, and we leave it to future work.

\subsubsection*{Reproducibility statement}
All claims in the main text have complete proofs in Appendices~\ref{app:proofs} and~\ref{app:lower}. The experiments of Section~\ref{sec:experiments} are described in full in Appendix~\ref{app:experiments}. 

\bibliographystyle{iclr2026/iclr2026_conference}
\bibliography{references}

@article{Agarwal_Kakade_etal_2020,
  author = {Agarwal, Alekh and Kakade, Sham M. and Lee, Jason D. and Mahajan, Gaurav},
  title = {On the Theory of Policy Gradient Methods: Optimality, Approximation, and Distribution Shift},
  number = {arXiv:1908.00261},
  year = {2020},
  publisher = {arXiv},
  url = {http://arxiv.org/abs/1908.00261},
  doi = {10.48550/arXiv.1908.00261},
}

@article{Alfano_Yuan_etal_2024,
  author = {Alfano, Carlo and Yuan, Rui and Rebeschini, Patrick},
  title = {A Novel Framework for Policy Mirror Descent with General Parameterization and Linear Convergence},
  number = {arXiv:2301.13139},
  year = {2024},
  publisher = {arXiv},
  url = {http://arxiv.org/abs/2301.13139},
  doi = {10.48550/arXiv.2301.13139},
}

@inproceedings{Altman_2021,
  author = {Altman, Eitan},
  title = {Constrained Markov Decision Processes: Stochastic Modeling},
  year = {2021},
  publisher = {Routledge},
  url = {https://www.taylorfrancis.com/books/9781315140223},
  doi = {10.1201/9781315140223},
  address = {Boca Raton}
}

@article{Besbes_Gur_etal_2015,
  author = {Besbes, O. and Gur, Y. and Zeevi, A.},
  title = {Non-stationary Stochastic Optimization},
  journal = {Operations Research},
  volume = {63},
  number = {5},
  year = {2015},
  pages = {1227--1244},
  url = {http://arxiv.org/abs/1307.5449},
  doi = {10.1287/opre.2015.1408},
}

@article{Burke_Ferris_1993,
  author = {Burke, J. V. and Ferris, M. C.},
  title = {Weak Sharp Minima in Mathematical Programming},
  journal = {SIAM Journal on Control and Optimization},
  volume = {31},
  number = {5},
  year = {1993},
  pages = {1340--1359},
  publisher = {Society for Industrial and Applied Mathematics},
  url = {https://epubs.siam.org/doi/10.1137/0331063},
  doi = {10.1137/0331063},
}

@misc{Cardoso_Wang_etal_2019,
  author = {Cardoso, Adrian Rivera and Wang, He and Xu, Huan},
  title = {Large Scale Markov Decision Processes with Changing Rewards},
  year = {2019},
  url = {https://arxiv.org/abs/1905.10649v1},
}

@book{CesaBianchi_Lugosi_2006,
  author = {Cesa-Bianchi, Nicol{\`o} and Lugosi, G{\'a}bor},
  title = {Prediction, Learning, and Games},
  year = {2006},
  doi = {10.1017/CBO9780511546921},
}

@misc{Cheung_Simchi-Levi_etal_2020,
  author = {Cheung, Wang Chi and Simchi-Levi, David and Zhu, Ruihao},
  title = {Reinforcement Learning for Non-Stationary Markov Decision Processes: The Blessing of (More) Optimism},
  year = {2020},
  url = {https://arxiv.org/abs/2006.14389v1},
}

@article{Dadashi_Taiga_etal_2019,
  author = {Dadashi, Robert and Ta{\"\i}ga, Adrien Ali and Roux, Nicolas Le and Schuurmans, Dale and Bellemare, Marc G.},
  title = {The Value Function Polytope in Reinforcement Learning},
  number = {arXiv:1901.11524},
  year = {2019},
  publisher = {arXiv},
  url = {http://arxiv.org/abs/1901.11524},
  doi = {10.48550/arXiv.1901.11524},
}

@article{Dick_Gyorgy_etal_2015,
  author = {Dick, Travis and Gyorgy, Andras and Szepesvari, Csaba},
  title = {Online Learning in Markov Decision Processes with Changing Cost Sequences},
  year = {2015},
  doi = {10.14288/1.0044649},
}

@article{Eshraghi_Liang_2022,
  author = {Eshraghi, Nima and Liang, Ben},
  title = {Dynamic Regret of Online Mirror Descent for Relatively Smooth Convex Cost Functions},
  number = {arXiv:2202.12843},
  year = {2022},
  publisher = {arXiv},
  url = {http://arxiv.org/abs/2202.12843},
  doi = {10.48550/arXiv.2202.12843},
}

@article{EvenDar_Kakade_etal_2009,
  author = {Even-Dar, Eyal and Kakade, Sham. M. and Mansour, Yishay},
  title = {Online Markov Decision Processes},
  journal = {Mathematics of Operations Research},
  volume = {34},
  number = {3},
  year = {2009},
  pages = {726--736},
  publisher = {INFORMS},
  url = {https://www.jstor.org/stable/40538442},
}

@article{Fei_Yang_etal_2020,
  author = {Fei, Yingjie and Yang, Zhuoran and Wang, Zhaoran and Xie, Qiaomin},
  title = {Dynamic Regret of Policy Optimization in Non-stationary Environments},
  number = {arXiv:2007.00148},
  year = {2020},
  publisher = {arXiv},
  url = {http://arxiv.org/abs/2007.00148},
  doi = {10.48550/arXiv.2007.00148},
}

@misc{Hall_Willett_2013,
  author = {Hall, Eric C. and Willett, Rebecca M.},
  title = {Dynamical Models and Tracking Regret in Online Convex Programming},
  year = {2013},
  url = {https://arxiv.org/abs/1301.1254v1},
}

@article{Jadbabaie_Rakhlin_etal_2015,
  author = {Jadbabaie, Ali and Rakhlin, Alexander and Shahrampour, Shahin and Sridharan, Karthik},
  title = {Online Optimization: Competing with Dynamic Comparators},
  number = {arXiv:1501.06225},
  year = {2015},
  publisher = {arXiv},
  url = {http://arxiv.org/abs/1501.06225},
  doi = {10.48550/arXiv.1501.06225},
}

@misc{Jin_Jin_etal_2019,
  author = {Jin, Chi and Jin, Tiancheng and Luo, Haipeng and Sra, Suvrit and Yu, Tiancheng},
  title = {Learning Adversarial MDPs with Bandit Feedback and Unknown Transition},
  year = {2019},
  url = {https://arxiv.org/abs/1912.01192v5},
}

@inproceedings{Kakade_Langford_2002,
  author = {Kakade, S. and Langford, J.},
  title = {Approximately Optimal Approximate Reinforcement Learning},
  year = {2002},
  url = {https://www.semanticscholar.org/paper/Approximately-Optimal-Approximate-Reinforcement-Kakade-Langford/523b4ce1c2a1336962444abc1dec215756c2f3e6},
}

@inproceedings{Kakade_2001,
  author = {Kakade, Sham M},
  title = {A Natural Policy Gradient},
  booktitle = {Advances in Neural Information Processing Systems},
  volume = {14},
  year = {2001},
  publisher = {MIT Press},
  url = {https://proceedings.neurips.cc/paper_files/paper/2001/hash/4b86abe48d358ecf194c56c69108433e-Abstract.html},
}

@article{Kleuker_Plaat_etal_2025,
  author = {Kleuker, Jan Felix and Plaat, Aske and Moerland, Thomas},
  title = {On the Effect of Regularization in Policy Mirror Descent},
  number = {arXiv:2507.08718},
  year = {2025},
  publisher = {arXiv},
  url = {http://arxiv.org/abs/2507.08718},
  doi = {10.48550/arXiv.2507.08718},
}

@article{Lan_2022,
  author = {Lan, Guanghui},
  title = {Policy Mirror Descent for Reinforcement Learning: Linear Convergence, New Sampling Complexity, and Generalized Problem Classes},
  number = {arXiv:2102.00135},
  year = {2022},
  publisher = {arXiv},
  url = {http://arxiv.org/abs/2102.00135},
  doi = {10.48550/arXiv.2102.00135},
}

@article{Li_Zhao_etal_2023,
  author = {Li, Long-Fei and Zhao, Peng and Zhou, Zhi-Hua},
  title = {Dynamic Regret of Adversarial Linear Mixture MDPs},
  journal = {Advances in Neural Information Processing Systems 36},
  year = {2023},
  pages = {60685–60711},
  publisher = {Neural Information Processing Systems Foundation, Inc. (NeurIPS)},
  url = {https://papers.nips.cc/paper_files/paper/2023/file/becd02b89259774da2ede23116a80648-Paper-Conference.pdf},
  doi = {10.52202/075280-2650},
  address = {New Orleans, Louisiana, USA}
}

@misc{Li_Zhao_etal_2024a,
  author = {Li, Long-Fei and Zhao, Peng and Zhou, Zhi-Hua},
  title = {Improved Algorithm for Adversarial Linear Mixture MDPs with Bandit Feedback and Unknown Transition},
  year = {2024},
  url = {https://arxiv.org/abs/2403.04568v1},
}

@misc{Li_Zhao_etal_2024,
  author = {Li, Long-Fei and Zhao, Peng and Zhou, Zhi-Hua},
  title = {Near-Optimal Dynamic Regret for Adversarial Linear Mixture MDPs},
  year = {2024},
  url = {https://arxiv.org/abs/2411.03107v1},
}

@article{Lu_Robinson_2008,
  author = {Lu, Shu and Robinson, Stephen M.},
  title = {Normal Fans of Polyhedral Convex Sets},
  journal = {Set-Valued Analysis},
  volume = {16},
  number = {2},
  year = {2008},
  pages = {281--305},
  url = {https://doi.org/10.1007/s11228-008-0077-9},
  doi = {10.1007/s11228-008-0077-9},
}

@article{Mao_Zhang_etal_2022,
  author = {Mao, Weichao and Zhang, Kaiqing and Zhu, Ruihao and Simchi-Levi, David and Ba{\c s}ar, Tamer},
  title = {Model-Free Non-Stationary RL: Near-Optimal Regret and Applications in Multi-Agent RL and Inventory Control},
  number = {arXiv:2010.03161},
  year = {2022},
  publisher = {arXiv},
  url = {http://arxiv.org/abs/2010.03161},
  doi = {10.48550/arXiv.2010.03161},
}

@article{Milosevic_Scherf_2025,
  author = {Milosevic, Nikola and Scherf, Nico},
  title = {The Geometry of Nonlinear Reinforcement Learning},
  number = {arXiv:2509.01432},
  year = {2025},
  publisher = {arXiv},
  url = {http://arxiv.org/abs/2509.01432},
  doi = {10.48550/arXiv.2509.01432},
}

@article{Muller_Montufar_2024,
  author = {M{\"u}ller, Johannes and Mont{\'u}far, Guido},
  title = {Geometry and convergence of natural policy gradient methods},
  journal = {Information Geometry},
  volume = {7},
  number = {S1},
  year = {2024},
  pages = {485--523},
  url = {http://arxiv.org/abs/2211.02105},
  doi = {10.1007/s41884-023-00106-z},
}

@book{Nemirovskii_IUdin_1983,
  author = {Nemirovski{\u\i}, Arkadi{\u\i} Semenovich and IUdin, David Berkovich},
  title = {Problem Complexity and Method Efficiency in Optimization},
  year = {1983},
  publisher = {Wiley},
}

@inproceedings{Neu_Antos_etal_2010,
  author = {Neu, Gergely and Antos, Andras and Gy{\"o}rgy, Andr{\'a}s and Szepesv{\'a}ri, Csaba},
  title = {Online Markov Decision Processes under Bandit Feedback},
  booktitle = {Advances in Neural Information Processing Systems},
  volume = {23},
  year = {2010},
  publisher = {Curran Associates, Inc.},
  url = {https://papers.nips.cc/paper_files/paper/2010/hash/7bb060764a818184ebb1cc0d43d382aa-Abstract.html},
}

@article{Paffenholz_2010,
  author = {Paffenholz, Andreas},
  title = {Polyhedral Geometry and Linear Optimization},
  year = {2010},
}

@book{Puterman_2014,
  author = {Puterman, Martin L.},
  title = {Markov Decision Processes: Discrete Stochastic Dynamic Programming},
  year = {2014},
  publisher = {John Wiley \& Sons},
}

@article{Rakhlin_Sridharan_2014,
  author = {Rakhlin, Alexander and Sridharan, Karthik},
  title = {Online Learning with Predictable Sequences},
  number = {arXiv:1208.3728},
  year = {2014},
  publisher = {arXiv},
  url = {http://arxiv.org/abs/1208.3728},
  doi = {10.48550/arXiv.1208.3728},
}

@book{Rockafellar_1970,
  author = {Rockafellar, R. Tyrrell},
  title = {Convex Analysis},
  year = {1970},
  publisher = {Princeton University Press},
  url = {https://www.jstor.org/stable/j.ctt14bs1ff},
}

@misc{Rosenberg_Mansour_2019,
  author = {Rosenberg, Aviv and Mansour, Yishay},
  title = {Online Convex Optimization in Adversarial Markov Decision Processes},
  year = {2019},
  url = {https://arxiv.org/abs/1905.07773v1},
}

@article{Schlaginhaufen_Kamgarpour_2023,
  author = {Schlaginhaufen, Andreas and Kamgarpour, Maryam},
  title = {Identifiability and Generalizability in Constrained Inverse Reinforcement Learning},
  number = {arXiv:2306.00629},
  year = {2023},
  publisher = {arXiv},
  url = {http://arxiv.org/abs/2306.00629},
  doi = {10.48550/arXiv.2306.00629},
}

@article{Schulman_Levine_etal_2017,
  author = {Schulman, John and Levine, Sergey and Moritz, Philipp and Jordan, Michael I. and Abbeel, Pieter},
  title = {Trust Region Policy Optimization},
  number = {arXiv:1502.05477},
  year = {2017},
  publisher = {arXiv},
  url = {http://arxiv.org/abs/1502.05477},
  doi = {10.48550/arXiv.1502.05477},
}

@article{Tomar_Shani_etal_2021,
  author = {Tomar, Manan and Shani, Lior and Efroni, Yonathan and Ghavamzadeh, Mohammad},
  title = {Mirror Descent Policy Optimization},
  number = {arXiv:2005.09814},
  year = {2021},
  publisher = {arXiv},
  url = {http://arxiv.org/abs/2005.09814},
  doi = {10.48550/arXiv.2005.09814},
}

@article{Tropp_2022,
  author = {Tropp, Joel A.},
  title = {ACM 204: Lectures on Convex Geometry},
  year = {2022},
  publisher = {California Institute of Technology},
  url = {https://resolver.caltech.edu/CaltechAUTHORS:20220412-220319430},
  doi = {10.7907/GEDA-H205},
}

@article{Wei_Luo_2021,
  author = {Wei, Chen-Yu and Luo, Haipeng},
  title = {Non-stationary Reinforcement Learning without Prior Knowledge: An Optimal Black-box Approach},
  number = {arXiv:2102.05406},
  year = {2021},
  publisher = {arXiv},
  url = {http://arxiv.org/abs/2102.05406},
  doi = {10.48550/arXiv.2102.05406},
}

@article{Zhan_Cen_etal_2023,
  author = {Zhan, Wenhao and Cen, Shicong and Huang, Baihe and Chen, Yuxin and Lee, Jason D. and Chi, Yuejie},
  title = {Policy Mirror Descent for Regularized Reinforcement Learning: A Generalized Framework with Linear Convergence},
  number = {arXiv:2105.11066},
  year = {2023},
  publisher = {arXiv},
  url = {http://arxiv.org/abs/2105.11066},
  doi = {10.48550/arXiv.2105.11066},
}

@misc{Zhao_Li_etal_2022,
  author = {Zhao, Peng and Li, Long-Fei and Zhou, Zhi-Hua},
  title = {Dynamic Regret of Online Markov Decision Processes},
  year = {2022},
  url = {https://arxiv.org/abs/2208.12483v1},
}

@article{Zhao_Zhang_etal_2020,
  author = {Zhao, Peng and Zhang, Yu-Jie and Zhang, Lijun and Zhou, Zhi-Hua},
  title = {Dynamic Regret of Convex and Smooth Functions},
  number = {arXiv:2007.03479},
  year = {2020},
  publisher = {arXiv},
  url = {http://arxiv.org/abs/2007.03479},
  doi = {10.48550/arXiv.2007.03479},
}

@article{Zhao_Zhang_etal_2024,
  author = {Zhao, Peng and Zhang, Yu-Jie and Zhang, Lijun and Zhou, Zhi-Hua},
  title = {Adaptivity and Non-stationarity: Problem-dependent Dynamic Regret for Online Convex Optimization},
  number = {arXiv:2112.14368},
  year = {2024},
  publisher = {arXiv},
  url = {http://arxiv.org/abs/2112.14368},
  doi = {10.48550/arXiv.2112.14368},
}

@book{Ziegler_1995,
  author = {Ziegler, G{\"u}nter M.},
  title = {Lectures on Polytopes},
  volume = {152},
  year = {1995},
  publisher = {Springer},
  url = {http://link.springer.com/10.1007/978-1-4613-8431-1},
  doi = {10.1007/978-1-4613-8431-1},
  series = {Graduate Texts in Mathematics},
  address = {New York, NY}
}

@inproceedings{Zimin_Neu_2013,
  author = {Zimin, Alexander and Neu, Gergely},
  title = {Online learning in episodic Markovian decision processes by relative entropy policy search},
  booktitle = {Advances in Neural Information Processing Systems},
  volume = {26},
  year = {2013},
  publisher = {Curran Associates, Inc.},
  url = {https://papers.neurips.cc/paper_files/paper/2013/hash/68053af2923e00204c3ca7c6a3150cf7-Abstract.html},
}

@article{Zinkevich_2003,
  author = {Zinkevich, Martin},
  title = {Online Convex Programming and Generalized Infinitesimal Gradient Ascent},
  year = {2003},
}

@article{Beck_Teboulle_2003,
  author = {Beck, Amir and Teboulle, Marc},
  title = {Mirror Descent and Nonlinear Projected Subgradient Methods for Convex Optimization},
  journal = {Operations Research Letters},
  volume = {31},
  number = {3},
  year = {2003},
  pages = {167--175},
  publisher = {Elsevier},
}

@article{Huleihel_Pal_etal_2021,
  author = {Huleihel, Wasim and Pal, Soumyabrata and Shayevitz, Ofer},
  title = {Learning User Preferences in Non-Stationary Environments},
  number = {arXiv:2101.12506},
  year = {2021},
  publisher = {arXiv},
  url = {http://arxiv.org/abs/2101.12506},
  doi = {10.48550/arXiv.2101.12506},
}

@inproceedings{Schafer_Konstan_etal_1999,
  author = {Schafer, J. Ben and Konstan, Joseph and Riedl, John},
  title = {Recommender systems in e-commerce},
  booktitle = {Proceedings of the 1st ACM conference on Electronic commerce},
  year = {1999},
  pages = {158–166},
  publisher = {Association for Computing Machinery},
  url = {https://dl.acm.org/doi/10.1145/336992.337035},
  doi = {10.1145/336992.337035},
  series = {EC ’99},
  address = {New York, NY, USA}
}

@article{Shani_Brafman_etal_2015,
  author = {Shani, Guy and Brafman, Ronen I. and Heckerman, David},
  title = {An MDP-based Recommender System},
  number = {arXiv:1301.0600},
  year = {2015},
  publisher = {arXiv},
  url = {http://arxiv.org/abs/1301.0600},
  doi = {10.48550/arXiv.1301.0600},
}

@book{Sutton_Barto_2018,
  author = {Sutton, Richard S. and Barto, Andrew G.},
  title = {Reinforcement Learning: An Introduction},
  year = {2018},
  publisher = {MIT Press},
}

\appendix

\section{Proofs from the main text}
\label{app:proofs}

\subsection{Within-face selection}

When the previous optimal face $F_{t-1}$ is higher-dimensional, the learner chooses one occupancy inside it, and that choice is the source of the selection error $\varepsilon_t^{\rm sel}$.

\begin{proposition}[Predictive selection]
	\label{prop:predsel}
	Let $m_t$ be a prediction of $\ell_t$ and choose $x_t\in\argmin_{u\in F_{t-1}}\ip{u}{m_t}$. If $\delta_t=\norm{m_t-\ell_t}_*$, then $\varepsilon_t^{\rm sel}\le\diam(F_{t-1})\,\delta_t$.
\end{proposition}

A learner that keeps the directions the old face leaves undetermined available for a later loss uses a Bregman selector $x_t\in\argmin_{u\in F_{t-1}}\bigl\{\ip{u}{m_t}+\lambda D_\psi(u,z_t)\bigr\}$, with $z_t$ the previous representative and $D_\psi$ a Bregman divergence \citep{Nemirovskii_IUdin_1983,Beck_Teboulle_2003}.

\begin{proposition}[Bregman face selection]
	\label{prop:bregsel}
	With $R_t^\psi=\sup_{u,v\in F_{t-1}}|D_\psi(u,z_t)-D_\psi(v,z_t)|$, the Bregman selector satisfies $\varepsilon_t^{\rm sel}\le\diam(F_{t-1})\,\delta_t+\lambda R_t^\psi$.
\end{proposition}

\begin{proof}[Proof of Propositions~\ref{prop:predsel} and~\ref{prop:bregsel}]
	Each selector is controlled by its own optimality condition, which bounds the loss gap between its chosen point and the best point of the face. Let $u_t^*\in\argmin_{u\in F_{t-1}}\ip{u}{\ell_t}$. For the predictive selector, optimality of $x_t$ for $m_t$ gives $\ip{x_t-u_t^*}{m_t}\le0$, so
	\[
		\varepsilon_t^{\rm sel}=\ip{x_t-u_t^*}{\ell_t}\le\ip{x_t-u_t^*}{\ell_t-m_t}\le\norm{x_t-u_t^*}\,\delta_t\le\diam(F_{t-1})\,\delta_t.
	\]
	For the Bregman selector the optimality condition carries the extra regularizer, giving $\ip{x_t-u_t^*}{m_t}\le\lambda(D_\psi(u_t^*,z_t)-D_\psi(x_t,z_t))\le\lambda R_t^\psi$, and the same chain of inequalities adds $\lambda R_t^\psi$ to the bound.
\end{proof}

\subsection{The cone detector}

\begin{proposition}[Prediction error matters only near walls]
	\label{prop:conedetect}
	Let $m_t$ be a predicted loss with $\sup_\pi|J_{\ell_t}(\pi)-J_{m_t}(\pi)|\le\varepsilon_t$, and let $\widehat\pi_t$ minimize $J_{m_t}$ at a deterministic policy, so its occupancy is a vertex of $\Q$. Then $J_{\ell_t}(\widehat\pi_t)-J_{\ell_t}^*\le 2\varepsilon_t$. Let $\Delta_t=\min\{\,\ip{v}{\ell_t}-J_{\ell_t}^*:\ v\text{ a vertex of }\Q,\ v\notin F_t\,\}$ be the smallest suboptimality among vertices outside the optimal face, the weak-sharpness margin of the current cone. If $\Delta_t>2\varepsilon_t$ then $\widehat\pi_t\in\Pi_t^*$ and the round contributes no regret.
\end{proposition}

\begin{proof}[Proof of Proposition~\ref{prop:conedetect}]
	The argument passes through the uniform value bound twice, once from the model to the truth and once back, after which the cone margin decides whether the model optimum can select the wrong face. Optimality of $\widehat\pi_t$ for $m_t$ gives $J_{m_t}(\widehat\pi_t)\le J_{m_t}(\pi_t^*)$ for any $\pi_t^*\in\Pi_t^*$, and two applications of the uniform bound $\sup_\pi|J_{\ell_t}(\pi)-J_{m_t}(\pi)|\le\varepsilon_t$ give $J_{\ell_t}(\widehat\pi_t)\le J_{m_t}(\widehat\pi_t)+\varepsilon_t\le J_{m_t}(\pi_t^*)+\varepsilon_t\le J_{\ell_t}(\pi_t^*)+2\varepsilon_t$. The occupancy of $\widehat\pi_t$ is a vertex, so it is either in $F_t$ or suboptimal by at least $\Delta_t$. When $\Delta_t$ exceeds $2\varepsilon_t$ the second case contradicts the bound just shown, so $\widehat\pi_t\in\Pi_t^*$ and the round contributes no regret.
\end{proof}

\subsection{The advantage-thresholded update}
\label{app:gate}

\begin{proposition}[Effect of the penalty]
	\label{prop:gate}
	Run the committed thresholded update under full information, restricting the play to $F_{t-1}$ through the selector of Proposition~\ref{prop:bregsel}, with an entropy regularizer and predictions of value error $\varepsilon_t$, and let $\Delta_t$ be the cone margin of Proposition~\ref{prop:conedetect} and $B=\max_{q,q'\in\Q}\ip{q-q'}{\ell_t}$ the range of the objective over $\Q$. Then
	\[
		\DReg_T \le r_1 + \sum_{t=2}^T\Gamma_t^{\mathrm{cross}} + B\sum_{t=2}^T\mathbf 1\{2\varepsilon_t\ge\Delta_t\} + \!\!\sum_{t:\,2\varepsilon_t<\Delta_t}\!\!\bigl(\diam(F_{t-1})\,\delta_t+\lambda R_t^\psi\bigr),
	\]
	so on every round whose margin the prediction resolves ($2\varepsilon_t<\Delta_t$) the tracker adds nothing to the intrinsic price beyond the controllable selection term. The penalty is not required for this bound, which holds for the unpenalized selector as well. At a fixed step size $\eta$ it nonetheless accelerates the tracker by the causal factor of Corollary~\ref{cor:causal}. To see this, take a confidently suboptimal action at layer $h$ with optimal advantage $\beta$ and on-policy gap only $g=\beta/(H-h+1)$. The penalty suppresses that action at rate $\eta(g+\lambda)$, against rate $\eta g$ for a plain mirror step. Its contribution to regret is therefore smaller by the factor $\Omega\bigl(1+\lambda(H-h+1)/\beta\bigr)$, and the two rates coincide only as $\eta\to\infty$.
\end{proposition}

We work under full information with an optimistic model $m_t$ of value error $\varepsilon_t=\sup_\pi|J_{\ell_t}(\pi)-J_{m_t}(\pi)|$ and selection error $\delta_t=\norm{m_t-\ell_t}_*$, an entropy regularizer $\psi$ that is $1$-strongly convex on each simplex factor, and the committed selector $x_t\in\argmin_{u\in F_{t-1}}\ip{u}{m_t}$. Write $B=\max_{q,q'\in\Q}\ip{q-q'}{\ell_t}$ for the range of the objective.

\begin{proof}[Proof of the bound]
	This bound assembles the decomposition, the selection bound, and the cone detector, splitting the rounds by whether the prediction resolves the margin. The play stays in $F_{t-1}$, and Theorem~\ref{thm:decomp} gives $\DReg_T=r_1+\sum_t\Gamma_t^{\mathrm{cross}}+\sum_t\varepsilon_t^{\rm sel}$, with Proposition~\ref{prop:bregsel} bounding $\varepsilon_t^{\rm sel}\le\diam(F_{t-1})\delta_t+\lambda R_t^\psi$. On a round with $2\varepsilon_t<\Delta_t$, a deterministic model optimum $\widehat\pi_t$ lies in $\Pi_t^*$ by Proposition~\ref{prop:conedetect}, so $m_t$ and $\ell_t$ select the same face. If in addition $F_{t-1}\cap F_t\neq\emptyset$, the selector attains $J_t^*$ over $F_{t-1}$, which gives $\varepsilon_t^{\rm sel}=0$ and $\Gamma_t^{\mathrm{cross}}=0$, a free round. On a round with $2\varepsilon_t\ge\Delta_t$ we fall back on $r_t\le B$. Collecting the three contributions yields the stated inequality.
\end{proof}

\begin{proof}[Proof of the suppression rate]
	Comparing the penalized and plain exponential-weights updates \citep{CesaBianchi_Lugosi_2006} on one suboptimal action comes down to tracking the odds each keeps on it. Fix a state-step $(s,h)$ with optimal action $a^*$ and a suboptimal action $a^-$ of optimal advantage $A^*(a^-)=\beta$. Its on-policy one-step gap is then $g=Q^{\pi_t}(s,a^-)-Q^{\pi_t}(s,a^*)=\beta/(H-h+1)$, the causal weighting of Section~\ref{sec:bellman}. The entropy update sets $\pi_{t+1}(a\mid s)\propto\pi_t(a\mid s)e^{-\eta c(a)}$, with cost $c=Q^{\pi_t}$ for the plain step and $c=Q^{\pi_t}+\lambda\mathbf 1\{a\notin\Aset_{t,h}^{\rm near}(s)\}$ for the penalized step. Taking the ratio of the masses on $a^-$ and $a^*$ cancels both the partition function and the optimal action's zero penalty, so the odds $o_t$ on $a^-$ satisfy $o^{\rm md}_{t+1}=o_t e^{-\eta g}$ for the plain step and $o^{\rm pen}_{t+1}=o_t e^{-\eta(g+\lambda)}$ for the penalized step. With the optimal face held fixed, the per-round regret from $(s,h)$ is $\beta\,o_t/(1+o_t)\in[\tfrac12\beta\min(o_t,1),\beta\min(o_t,1)]$. Summing the geometric decays then gives totals proportional to $1/(1-e^{-\eta g})$ and $1/(1-e^{-\eta(g+\lambda)})$, whose ratio is $\Omega((g+\lambda)/g)=\Omega(1+\lambda(H-h+1)/\beta)$ whenever $\eta g\lesssim1$. As $\eta\to\infty$ both updates collapse to the hard $\argmin$, so the penalty alters only the finite-step transient and leaves the optimum unchanged.
\end{proof}

\begin{figure}[h]
	\centering
	\includegraphics[width=0.88\linewidth]{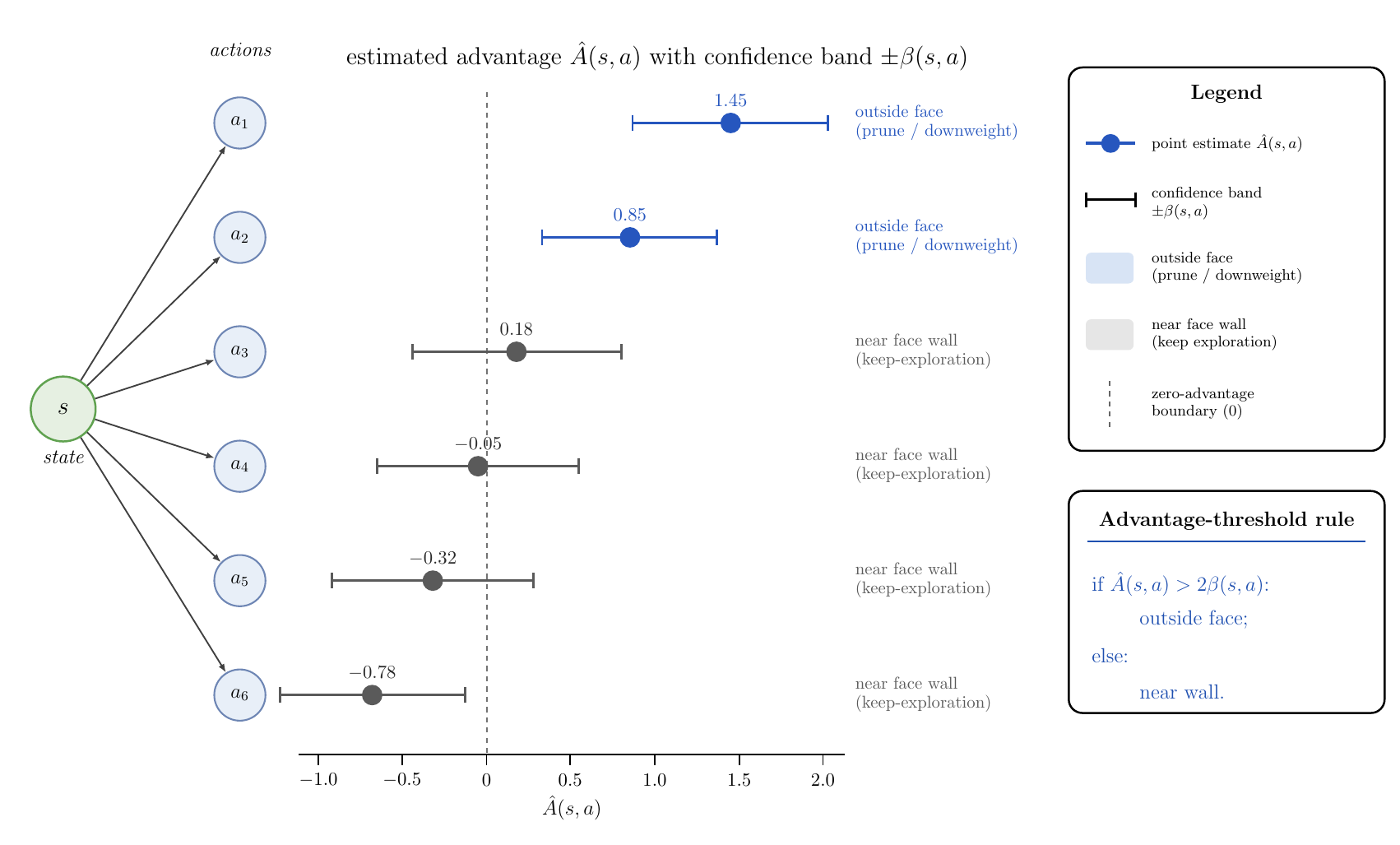}
	\caption{The advantage threshold of Section~\ref{sec:algorithms}. Actions whose estimated advantage exceeds their confidence band lie outside the optimal face and lose mass, while actions whose confidence band crosses zero lie near a wall of the fan and retain mass pending further evidence.}
	\label{fig:advantage-gate}
\end{figure}

\begin{remark}[Scope of Proposition~\ref{prop:gate}]
	The bound is a composition of Theorem~\ref{thm:decomp}, Proposition~\ref{prop:conedetect}, and Propositions~\ref{prop:predsel}--\ref{prop:bregsel}, and the penalty enters only through the suppression rate. Its uniform-margin corollary takes $2\varepsilon_t<\Delta_t$ on every round, which reduces the regret to $r_1+\Gamma_{2:T}^{\mathrm{cross}}$ in the vertex case. That corollary assumes the loss never lingers near a wall. Proposition~\ref{prop:lower} shows the lingering case is the hard one, and the reading is honest only for benign non-stationarity.
\end{remark}

\section{Separations and lower bounds}
\label{app:lower}

Proposition~\ref{prop:largevar} shows that loss variation can be arbitrarily large while the intrinsic price is zero. Whether a given crossing is free or priced is decided entirely by the relative position of the old and new optimal faces, as Figure~\ref{fig:degenerate} illustrates. The complementary phenomenon is that an unpredictable crossing forces a fixed positive price on every learner, whatever its strategy. This invariance is the precise sense in which the price is intrinsic.

\begin{figure}[h]
	\centering
	\includegraphics[width=0.72\linewidth]{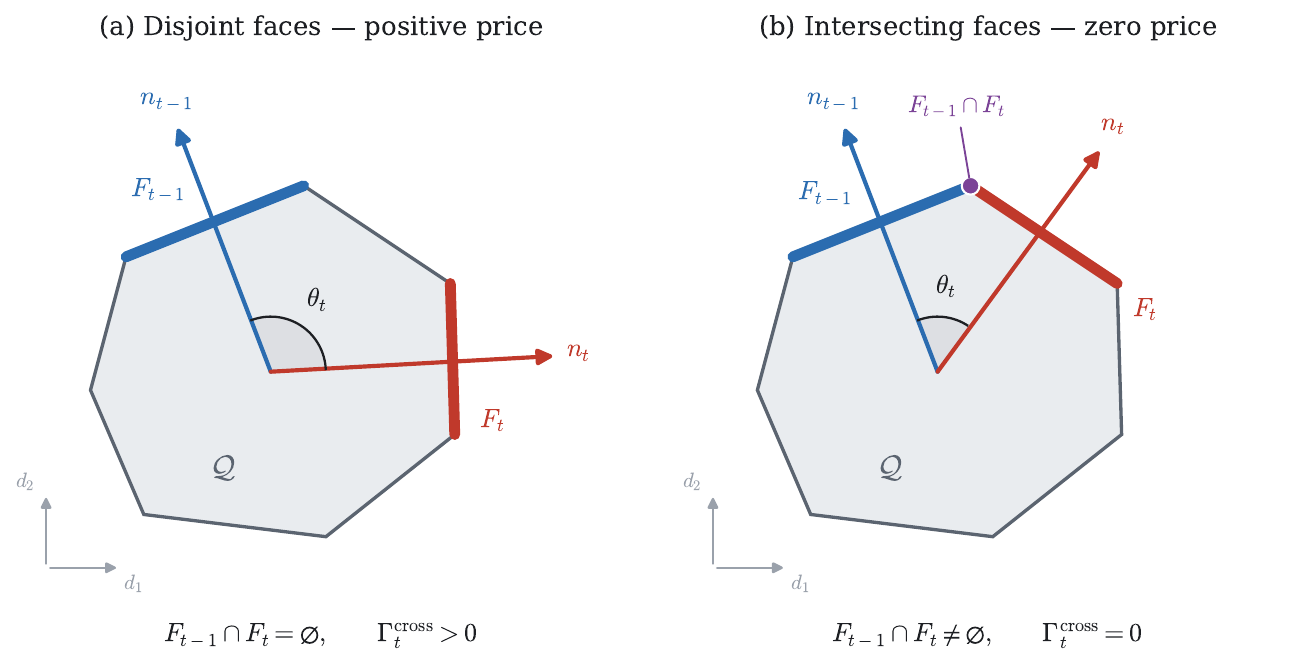}
	\caption{Whether a crossing is priced depends on the relative position of the old and new optimal faces. When the faces are disjoint (left) every old optimal occupancy is suboptimal under the new loss and the price is positive. When they intersect (right) some old occupancy remains optimal and the price is zero, regardless of how far the loss has moved.}
	\label{fig:degenerate}
\end{figure}

\begin{proposition}[Unpredictable priced crossings are unavoidable]
	\label{prop:lower}
	Consider a one-state, one-step MDP with two actions. At each round independently let $\ell_t=(0,\gamma)$ or $(\gamma,0)$ with equal probability. For any learner, even one that observes $\ell_t$ in full after acting,
	\[
		\E[\DReg_T]\ge\frac{\gamma T}{2},
	\]
	and the previous-face tracker attains this rate up to the first-round convention.
\end{proposition}

\begin{proof}
	Before observing $\ell_t$ the learner commits to a distribution $p_t$ over the two actions. Conditional on the past the optimal action is uniform and independent of $p_t$, so the expected loss of any choice is $\gamma/2$ while the optimal loss is $0$. Summing the per-round gap of $\gamma/2$ over $T$ rounds gives the bound. A learner that repeats the previous round's optimal action errs precisely when the sign flips, which happens with probability one half each round, matching the rate.
\end{proof}

The degenerate case adds a second irreducible term. Two loss sequences can present the same intrinsic price and yet impose different regret on a fixed within-face selector. The selection error of Theorem~\ref{thm:decomp} therefore cannot be absorbed into the price.

\begin{proposition}[Degenerate faces require selection]
	\label{prop:degsel}
	There exist two sequences of optimal faces with identical $\Gamma_t^{\mathrm{cross}}$ but different regret for a fixed selector $x_t\in F_{t-1}$. Hence $\Gamma_t^{\mathrm{cross}}$ alone cannot characterize algorithmic regret under degeneracy.
\end{proposition}

\begin{proof}[Proof sketch]
	Use a layered binary-tree MDP and let $F_{t-1}$ be the set of root-to-leaf paths sharing a fixed prefix, so the suffix is free. Construct two new faces $F_t$ and $F_t'$ that share the same best point inside $F_{t-1}$, hence the same $\Gamma_t^{\mathrm{cross}}$, but require opposite suffixes. The oracle point of $F_{t-1}$ pays the same price in both, while a fixed selector that committed to one suffix is cheap against the compatible face and expensive against the other. The selection error is therefore an independent degree of freedom.
\end{proof}

\subsection{Construction for Theorem~\ref{thm:varsep}}
\label{app:varsep}

A crossing forced at layer $h$ flips the suffix of length $H-h+1$ and changes the loss on exactly that many layers. Within a single layered tree the price and the variation therefore move together. A single-coordinate change fares no better in this regime, where a crossing of depth $D$ must first clear an optimality margin of order $D$ before the displaced face pays an amount of the same order. The separation therefore turns on the one asymmetry between the two quantities --- loss variation depends on the loss sequence alone and stays invariant to the MDP, whereas the face-crossing price reads the normal fan and so depends on the transitions. We hold the loss change fixed across two instances and let the geometry alone decide whether it reroutes a single step or the whole horizon. Both have horizon $H$, binary actions, and deterministic transitions from one start state, so a deterministic policy is a root-to-leaf action string. Throughout, fix $\gamma>0$ with a small $\kappa\in(0,\gamma)$, and set $R=\gamma(H-1)$.

\paragraph{The late instance.} A separable chain places one state at each of $H$ layers. At every non-final layer the action changes neither the cost nor the next state. A fixed preference of $\kappa$ for action $0$ enters there in both rounds, which keeps the round-$1$ optimum a single vertex with no degenerate face. The last layer costs $(\,0,\ R-\gamma\,)$ in round $1$ and $(\,0,\ -\gamma\,)$ in round $2$, where $R-\gamma=\gamma(H-2)\ge0$. Action $0$ is then uniquely optimal at every layer in round $1$, so $F_1$ is the single path taking action $0$ throughout. Action $1$ becomes optimal at the last layer in round $2$, yet the old path still takes action $0$ there and pays exactly $\gamma$. The only coordinate that changes is the last layer's action-$1$ entry, which moves from $R-\gamma$ to $-\gamma$, a change of magnitude $R$.

\paragraph{The early instance.} We can replace the chain with a branching tree. From the root, action $1$ leads to a toll node at layer $1$, while action $0$ commits to an expensive line that costs $\gamma$ per layer for the rest of the horizon. At the toll node, action $1$ pays a toll and unlocks a cheap line that costs $0$ thereafter, and a fixed bonus $\kappa$ on the committing root action again makes the round-$1$ optimum a unique vertex. Between rounds only one coordinate changes, the toll, which stands at $R=\gamma(H-1)$ in round $1$ and drops to $0$ in round $2$. The round-$1$ toll equals the whole-horizon cost of the expensive line, so the committed expensive path is optimal. Once the toll vanishes in round $2$, the new optimum routes through the toll node, and the old committed path now pays $\gamma$ on each of the $H-1$ expensive layers. The round-to-round change again touches a single coordinate, by the same magnitude $R$.

\begin{proof}[Proof of Theorem~\ref{thm:varsep}]
	In both instances the round-to-round loss change touches a single coordinate of magnitude $R$. A one-hot vector has the same $L_p$ norm for every $p\in[1,\infty]$, so $V_2^\ell=R$ under every dual norm. The early instance leaves its committed path at round-$2$ value $\gamma(H-1)-\kappa$ against a new optimum of $0$, giving $\Gamma_2^{\mathrm{cross}}=\gamma(H-1)-\kappa$. The late instance reverses the figures, its old path resting at $0$ against a new optimum of $-\gamma$, giving $\Gamma_2^{\mathrm{cross}}=\gamma$. Their ratio tends to $H-1$ as $\kappa\to0$. Both old faces are vertices, and Theorem~\ref{thm:decomp} equates the dynamic regret with the price on each instance. Any valid $B$ must therefore satisfy $B(R)\ge\gamma(H-1)-\kappa$ to clear the early instance. That same $B(R)$ then overshoots the late regret $\gamma$ by the factor $H-1-\kappa/\gamma$.
\end{proof}

\begin{remark}[Why two MDPs are needed]
	A defender of variation might still ask for a single fixed MDP carrying two loss sequences. A one-coordinate change there keeps the price-to-trigger ratio near a small constant. The collinearity above therefore forecloses that route. Figure~\ref{fig:barrier} traces this, the single-MDP ratio flattening near two as the horizon grows while the two-MDP construction reaches $H-1$. The two-MDP family therefore shows what a single MDP cannot, that the variation term stays blind to the transitions. Rather than computing them by hand, we read the prices above off the Bellman backup of Theorem~\ref{thm:bellman}, where they agree with the value-gap definition.
\end{remark}

\begin{figure}[t]
	\centering
	\includegraphics[width=0.6\linewidth]{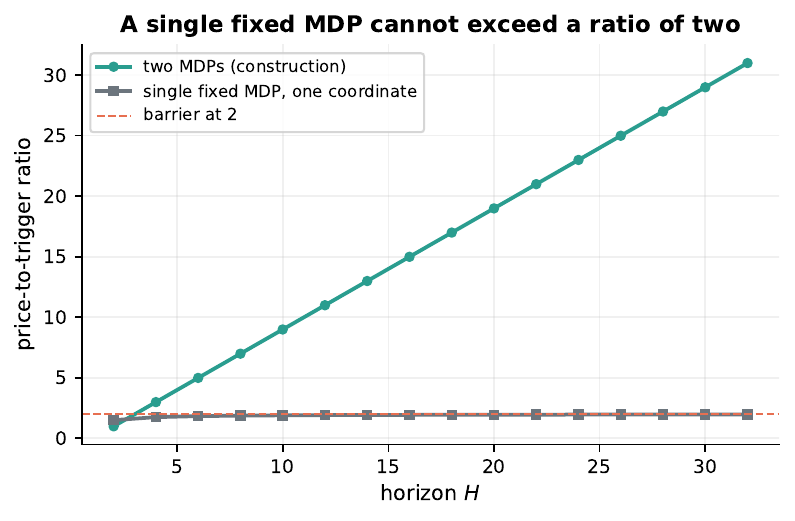}
	\caption{Why two MDPs are needed. Inside a single fixed MDP, a one-coordinate loss change that forces a crossing holds the price-to-trigger ratio near two at every horizon, whereas the two-MDP construction of Theorem~\ref{thm:varsep} reaches $H-1$. Both curves are computed by the Bellman backup of Theorem~\ref{thm:bellman}.}
	\label{fig:barrier}
\end{figure}

\section{Experimental details}
\label{app:experiments}

All experiments use fixed transitions and adversarial losses. We leave the moving-transition case aside. There the polytope $\Q$ deforms between rounds, so a face of one round need not remain feasible in the next and the fixed-fan decomposition no longer applies. Every figure averages $32$ seeds, and error bars or shaded bands show standard errors for the displayed means. The simplex, layered, and degenerate benchmarks compute the geometric quantities in closed form for comparison with the measured regret. 

\paragraph{Configuration.}
Every benchmark uses $32$ seeds, indices $0$ through $31$, and reports the best setting over the step-size and method-specific sweeps in Table~\ref{tab:hparams}. The per-step loss unit in the priced quantities is $\gamma$.

\begin{table}[h]
	\caption{Hyperparameters. Step sizes $\eta$ and the method-specific parameters are swept over the ranges shown, and every benchmark uses $32$ seeds.}
	\label{tab:hparams}
	\begin{center}
		\small
		\begin{tabular}{@{}p{0.15\linewidth}p{0.77\linewidth}@{}}
			\toprule
			\textbf{Benchmark} & \textbf{Settings} \\
			\midrule
			Simplex & $K\in\{16,32,64,128\}$ actions, horizon $T{=}5000$, $\gamma{=}0.25$, $\eta\in\{0.04,0.08,0.15,0.3,0.6,1.0\}$, $24$ trials per regime, four regimes \\
			Layered tree & $H\in\{8,12,16,20\}$, $\gamma{=}0.2$, sequence length $5000$, $96$ sequences per horizon, $\eta\in\{0.04,\ldots,1.0\}$ \\
			Degenerate & horizon $20$, $\gamma{=}0.2$, prefix length $1$ to $18$, sequence length $5000$, switch probability $0.025$, four selectors \\
			Gridworld & side $15$, horizon $30$, $3000$ episodes, $\eta\in\{0.04,\ldots,0.6\}$, advantage threshold $\in\{0.02,\ldots,0.18\}$, penalty weight $\in\{0.35,0.7\}$, trust-region $\alpha{=}0.12$, four scenarios \\
			\bottomrule
		\end{tabular}
	\end{center}
\end{table}

\paragraph{Simplex normal fan.}
With a single state and $K\in\{16,32,64,128\}$ actions, $\Q$ collapses to the simplex, its fan the arrangement of the normal cones. The loss runs through four regimes --- within-cone drift, small wall-crossing drift, random high-margin crossings, and a rotating optimum. Against it we compare Hedge \citep{CesaBianchi_Lugosi_2006}, Euclidean mirror descent \citep{Beck_Teboulle_2003}, optimistic Hedge \citep{Rakhlin_Sridharan_2014}, and the previous-face tracker. Figure~\ref{fig:exp-simplex} contrasts loss variation with the face price by regime, then plots the previous-face tracker's regret against $\Gamma_{2:T}^{\mathrm{cross}}$. The within-cone regime carries large variation at zero price, so any variation-based bound predicts regret where there is none, and across regimes the tracker's own regret correlates at $1.00$ with $\Gamma_{2:T}^{\mathrm{cross}}$ and only at $0.89$ with $V_T^\ell$.

\begin{figure}[h]
	\centering
	\includegraphics[width=0.92\linewidth]{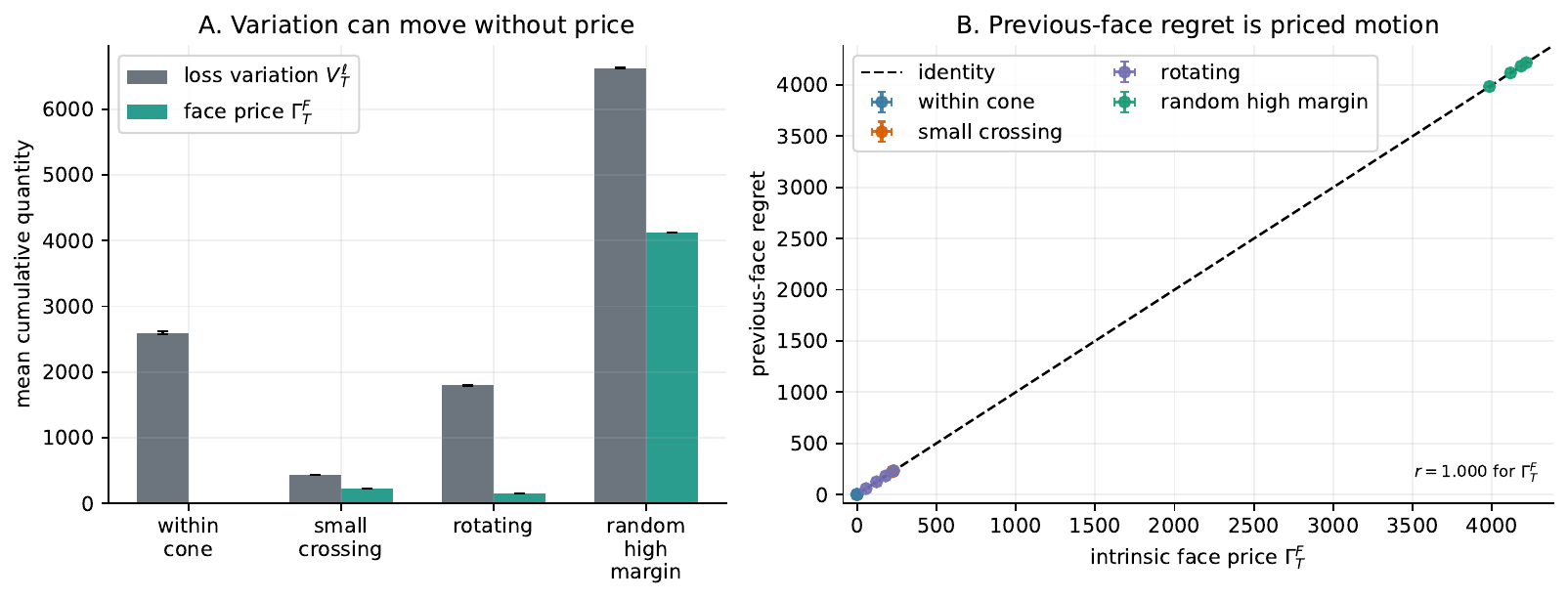}
	\caption{Simplex normal fan. Left, mean loss variation and intrinsic face price by regime, with uncertainty over trials. Right, previous-face tracker regret against $\Gamma_{2:T}^{\mathrm{cross}}$, where the identity line gives the exact priced-face-motion prediction. Error bars are standard errors over $32$ seeds and $24$ trials per action-size/regime.}
	\label{fig:exp-simplex}
\end{figure}

\paragraph{Causal anisotropy in a layered tree.}
A deterministic binary tree of horizon $H\in\{8,12,16,20\}$, in which each policy is a root-to-leaf path and one path is optimal. A crossing is forced at a chosen layer and we record $\Gamma_t^{\mathrm{cross}}/\gamma$. The identity itself is in the main text. Figure~\ref{fig:exp-layered-pred} compares loss variation, raw path length, causal path length, and $\Gamma_{2:T}^{\mathrm{cross}}$ as predictors of mirror-descent regret across the sequences. In this structured family all four explain nearly all of the regret variation, and the experiment that separates $\Gamma_{2:T}^{\mathrm{cross}}$ from loss variation is the simplex above.

\begin{figure}[h]
	\centering
	\includegraphics[width=0.6\linewidth]{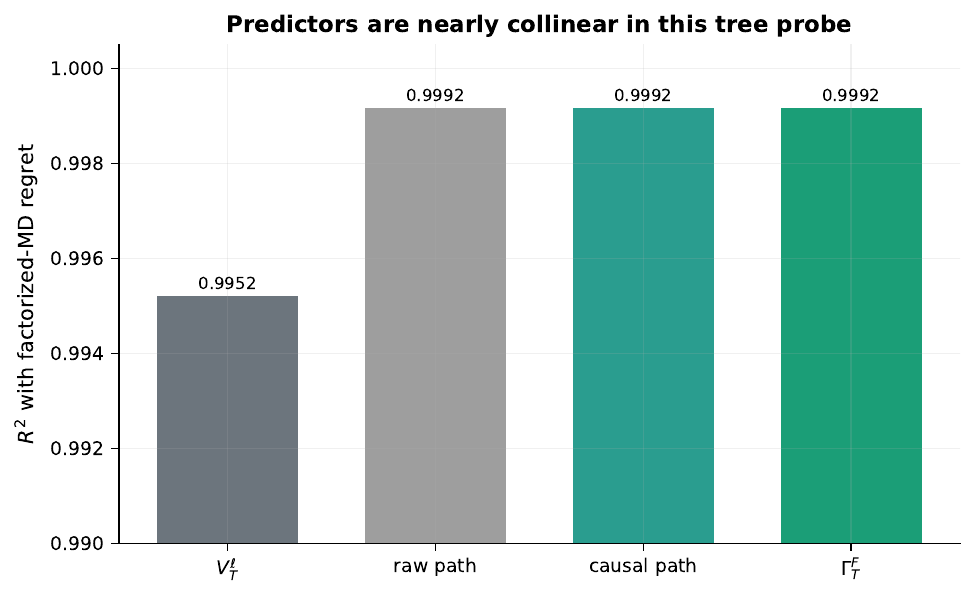}
	\caption{Layered tree. Predictive power of four candidate path-motion summaries for dynamic regret, using $32$ seeds and $96$ sequences per horizon. In this family the predictors move together, which is why the decisive separation between price and variation is made on the simplex rather than here.}
	\label{fig:exp-layered-pred}
\end{figure}

\paragraph{Degenerate optimal faces.}
All paths sharing a prefix make up the optimal set in this benchmark's layered tree, and those prefixes coarsen, refine, and conflict across $5000$ rounds as four within-face selectors compete. An \emph{oracle} observes $\ell_t$ and plays the loss-minimizing point $\argmin_{u\in F_{t-1}}\ip{u}{\ell_t}$, which achieves zero selection error by construction. The \emph{sticky} Bregman selector instead pulls toward its previous representative $z_t$ by minimizing $\ip{u}{m_t}+\lambda D_\psi(u,z_t)$, as in Proposition~\ref{prop:bregsel}. \emph{Lexicographic} selection applies a fixed, loss-independent tie-break over paths, while \emph{random} selection draws a uniform point of $F_{t-1}$. The main text reports the regret stack, and Figure~\ref{fig:exp-degen-dim} shows the selection error growing with the dimension of the free face, the component that the sticky projection controls.

\begin{figure}[h]
	\centering
	\includegraphics[width=0.6\linewidth]{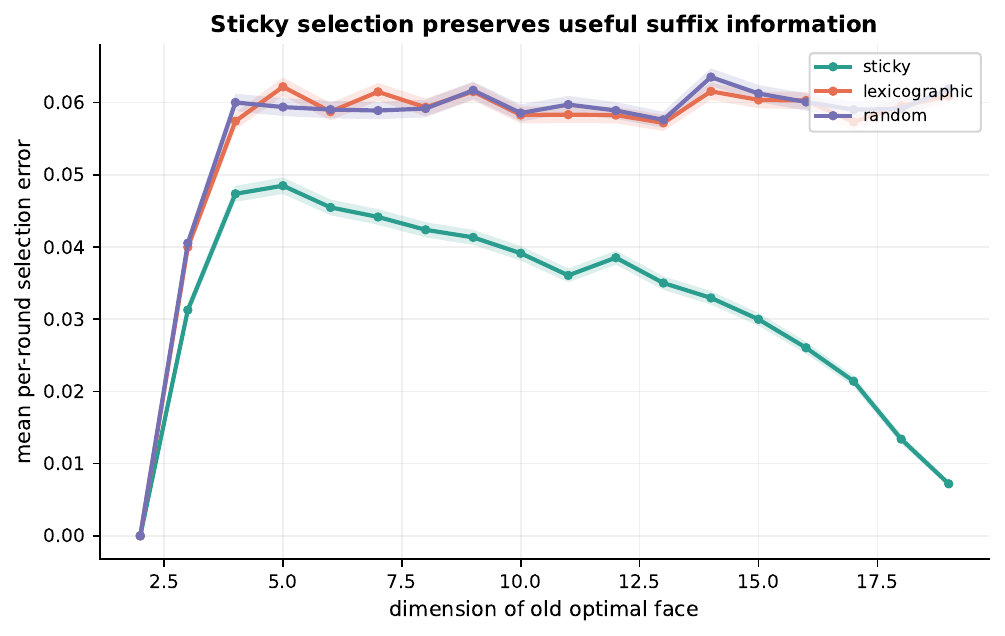}
	\caption{Degenerate faces. Within-face selection error against the dimension of the free face, over $32$ seeds and $5000$-round sequences. A larger free face leaves more room for a poor representative, and the sticky selector keeps the error below the lexicographic and random ones, with bands showing standard errors within each dimension.}
	\label{fig:exp-degen-dim}
\end{figure}

\paragraph{Advantage-thresholded gridworld control.}
Following the standard tabular gridworld convention \citep{Sutton_Barto_2018}, our custom non-stationary gridworlds of side $15$ and horizon $30$ supply four scenarios, a moving goal, a rotating field, changing obstacles, and a boundary crossing. Against them we run mirror descent \citep{Nemirovskii_IUdin_1983,Beck_Teboulle_2003}, a trust-region update \citep{Schulman_Levine_etal_2017}, optimistic mirror descent \citep{Rakhlin_Sridharan_2014}, and the advantage-thresholded update. For each we sweep step sizes, the advantage threshold, and the penalty weight, then report its best setting. Figure~\ref{fig:exp-grid} reports the best cumulative regret of every method, and the thresholded update comes in lowest in each scenario. Figure~\ref{fig:exp-grid-ts} then tracks per-episode regret together with the occupancy mass near the fan walls over training, where the thresholded update holds more mass near those walls in the scenarios whose optimum moves sharply.

\begin{figure}[h]
	\centering
	\includegraphics[width=0.78\linewidth]{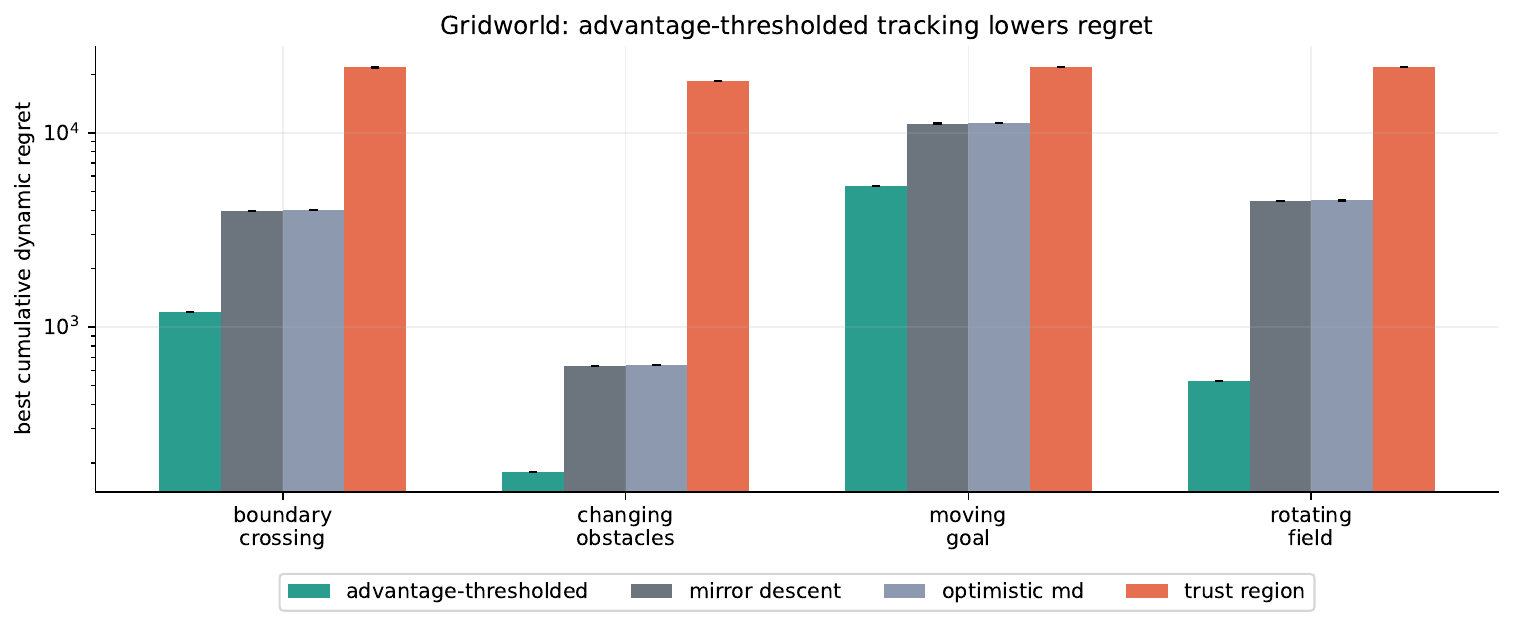}
	\caption{Best cumulative dynamic regret of four updates on four non-stationary gridworld scenarios, shown on a log scale. Bars are means over $32$ seeds with standard-error bars. The thresholded update is lowest in every scenario.}
	\label{fig:exp-grid}
\end{figure}

\begin{figure}[h]
	\centering
	\includegraphics[width=0.95\linewidth]{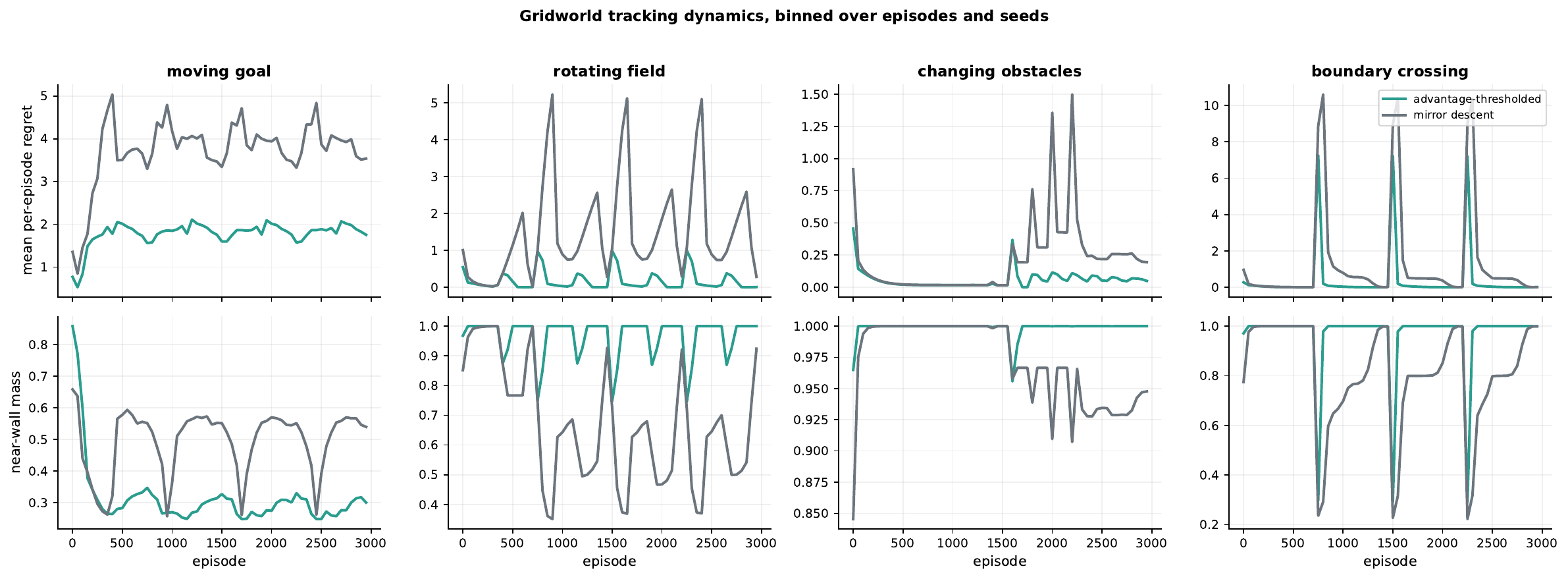}
	\caption{Gridworld over training, the thresholded update (teal) against mirror descent (gray), with curves binned over episodes and $32$ seeds. Rows, top to bottom, are per-episode dynamic regret and occupancy mass near the fan walls, across the four scenarios.}
	\label{fig:exp-grid-ts}
\end{figure}

\end{document}